\documentclass[11pt]{article}

\usepackage{mathrsfs,amsmath,amsfonts,amssymb,amstext,bm,bbm,dsfont,pifont,amscd,
            amsthm,stmaryrd,euscript,color,xcolor}

\usepackage{algorithmic}
\usepackage{algorithm}
\usepackage{url}
\usepackage{mathtools}
\usepackage{epsfig}
\usepackage{float}
\usepackage{appendix}
\usepackage{amstext}
\usepackage{floatflt}
\usepackage{nicefrac}
\usepackage{amsmath}
\usepackage{palatino}
\usepackage{anysize,hyperref}
\usepackage{enumerate}
\usepackage{epsfig}
\usepackage{hyperref}
\usepackage{bm}
\usepackage{bbm}
\usepackage{graphicx}
\numberwithin{equation}{section}

\input{macros_nicole}

\usepackage{mathtools}

\title{Stochastic Gradient Descent in Hilbert Scales:\\ Smoothness, Preconditioning and Earlier Stopping}

\author{Nicole M\"ucke\thanks{MATH+ Junior Research Group 
\emph{Mathematical Foundations of Data Science}}  \\
Technical University  Berlin \\
\texttt{muecke@math.tu-berlin.de} 
\and
Enrico Reiss\\ 
University of Potsdam \\ 
\texttt{enreiss@uni-potsdam.de}
}

\date{\today}

\begin{document}

\maketitle

\begin{abstract}
Stochastic Gradient Descent (SGD) has become the method of choice for solving a broad range of machine learning problems. 
However, some of its learning properties are still not fully understood. 
We consider least squares learning in reproducing kernel Hilbert spaces (RKHSs) and extend the classical SGD analysis to a learning setting in Hilbert scales, 
including Sobolev spaces and Diffusion spaces on compact Riemannian manifolds. 
We show that even for well-specified models, violation of a traditional  benchmark smoothness assumption has a tremendous effect  
on the learning rate. 
In addition, we show that for miss-specified models, preconditioning 
in an appropriate Hilbert scale helps to reduce the number of iterations, i.e. allowing for \emph{earlier stopping}. 
\end{abstract}

\section{Introduction}


When solving non-parametric least-squares problems in an RKHS we face the problem that the unknown solution may not have the expected 
smoothness  (regularity) implied by the kernel. 
Then the question arises whether the use of such \emph{mis-specified} kernels still allows for good reconstructions yielding errors of optimal order. 
Although it is a commonly accepted fact that the regularity inherent in the solution has an impact on accuracy and convergence of learning 
algorithms, there are only poor precise mathematical investigations in the framework of learning in RKHSs using SGD. 
Mathematically, \emph{smoothness} 
can be expressed in various different ways. Classically, the concept of \emph{source conditions} proved to be useful, expressing the target function as 
element of the domain of a differential operator, see e.g. \cite{engl96}, hence it can be \emph{differentiated} and is therefore considered to have a certain 
degree of smoothness. This approach can be extended to more 
\emph{general source conditions}, where the objective function belongs to the domain of a more general unbounded operator.   
In learning theory, a similar concept is also now state of the art.  
Here, a minimizer is considered as smooth if it belongs to the range of a function of the integral operator associated to the kernel of the RKHS. 
Several results showing optimal rates of convergence for different regularization algorithms 
are available under such a condition. In \cite{optimalratesRLS},  \cite{rastogi2017optimal}, \cite{fischer2017sobolev}, 
fast optimal rates of convergence have been established for kernel ridge regression. 
The authors in \cite{bauer2007regularization}, \cite{blanchard2018optimal} consider general spectral algorithms under different assumptions and  
these results have been extended to more general Hilbert spaces than RKHSs in \cite{Lin18}, to name just a few. 
However, none of these works investigates the learning properties of SGD under more general smoothness properties of the solution, i.e. going beyond the kernel integral 
operator.  

The literature about convergence properties in non-parametric statistical learning of SGD is also vast and different flavors have been investigated. 
Starting with \cite{SmaYao06} and then further investigated in \cite{YinPon08,YaoTar14,orabona, NemJudLan09}, one-pass SGD is considered, 
that is, each data point is used only once. An 
additional twist here is to establish learning rates by averaging the iterates, which is obtained in \cite{RSS12,HK14}, dating back to ideas 
in \cite{R88}, \cite{PolJu92}. 
More recently, averaging was shown to lead to larger, possibly
constant, step-sizes, see \cite{BacMou13,DieuBa16,DieFlaBac17}, but preventing optimal rates in the high smoothness regime. 
This has been alleviated in  \cite{mucke2019beating} by considering 
tail-averaging  with mini-batching. The role of mini-batching has also been considered and shown to
potentially lead to linear parallelization speedups in \cite{Cotter11} (and references therein). Additionally, there are some results investigating 
the role of multiple passes for learning \cite{RosVil15, HarRecSin16,LinCamRos16}. The authors in \cite{LinRos17} derive optimal results for 
multipass SGD  considering also the 
effect of mini-batching. Following the approach in this latter  paper,  multipass SGD with averaging was analyzed by  \cite{pillaud2018statistical} with no minibatching.

We extend those results to learning with SGD, expressing smoothness of the target in terms of a general \emph{Hilbert scale}, being a nested sequence of 
Hilbert spaces, generated by an unbounded operator, e.g. a differential operator. Thus, those spaces are natural candidates for representing a certain degree of regularity 
of the objective function. 
Historically, regularization in Hilbert scales was introduced in the context of statistical inverse problems in Hilbert spaces in e.g. 
\cite{natterer1984error}, \cite{mathe2007error}, \cite{tautenhahn1996error}, \cite{nair2005regularization}, \cite{mair1994tikhonov} 
to improve convergence rates if the objective function 
is very smooth. In contrast, if the objective function has only poor smoothness properties it turns out that it is sufficient to regularize in a weaker norm 
to obtain optimal rates of convergence, see also \cite{egger2005preconditioning}. A first attempt to introduce regularization in Hilbert scales 
in the framework of (linear) inverse learning theory has been accomplished in \cite{rastogi2020inverse} where general spectral algorithms are investigated, 
but excluding SGD performed directly in an RKHS. 

We fill these gaps and investigate the learning properties of tail-averaged Gradient Descent and tail-averaged mini-batch SGD with constant step-size in Hilbert scales under different 
smoothness assumptions. For \emph{well-specified models}, i.e. the solution 
belongs to the RKHS, we show that violation of a given \emph{benchmark smoothness} slows down convergence.  
Additionally we show that \emph{smoothness promoting SGD}, i.e. mapping the kernel into a smaller Hilbert space consisting of more regular functions 
improves convergence if the target is sufficiently smooth.   
For \emph{mis-specified models}, i.e. the objective function does not belong to the RKHS and is less regular than the kernel, we show that \emph{preconditioning}, that is, performing SGD in 
a larger Hilbert space with weaker norm, allows to reduce the number of iterations necessary for achieving the minimum. We furthermore investigate the interplay of all parameters involved, 
i.e. the step-size, mini-batch size, 
stopping time and smoothness. 
In this generality, our results are new in the learning theory framework of non-parametric regression with SGD. As a byproduct, we also show the benefit of tail-averaging 
in the high smoothness regime. Previous results are recovered as a special case.  
\paragraph{Outline} 
In Section \ref{sec-learning-in-hilbert-scales} we introduce the traditional learning setting in RKHSs. 
An introduction into the theory of Hilbert scales is given Section \ref{sec-learning-scales-theory} and we present our tail-averaged mini-batch SGD recursion. 
Section \ref{sec-smoothness-promoting} is devoted to presenting and discussing our main findings.  
Some numerical illustrations are given Section \ref{sec-numerics}. 
All proofs are deferred to the Appendix. 


\section{Learning in Reproducing Kernel Hilbert Spaces: State of the Art}
\label{sec-learning-in-hilbert-scales}
We consider a joint probability distribution $\rho$  on the input/ output pair $\cX \times \cY$, with $\cY\subseteq [-M,M]$, for some $M>0$. 
By $\rho_X$ we denote the marginal measure on $\cX$ and $\rho(\cdot|x)$ is the conditional distribution on $\cY$ given $x \in \cX$. 
In least squares regression, we aim at minimizing the \emph{expected risk}
\begin{equation}
\label{eq:min-prob}
 \inf_{f \in \cH} \cE(f) \;, \quad \cE(f)=\int_{\cX \times \cY}(f(x)-y)^2\; d\rho(x,y) \;, 
\end{equation} 
where $\cH$ is an appropriate \emph{hypothesis space}. We focus in particular on $\cH$ being a  
\emph{reproducing kernel Hilbert space} (RKHS), see e.g. \cite{aronszajn1950theory}, \cite{StCh08} 
arising from a kernel $K: \cX \times \cX \to \mbr$, satisfying: 
\begin{assumption}
\label{ass:bounded}
Suppose that $\sup_{x,x' \in \cX}K(x,x') \leq \kappa^2$, for some  $\kappa < \infty$, $\rho_X$-almost surely. 
\end{assumption}
Note that under this Assumption, the space $\cH$ can be continuously embedded into $L^2(\cX, \rho_X)$, the space of square integrable functions on 
$\cX$ with respect to $\rho_X$. In particular, we let 
$\cS: \cH \hookrightarrow   L^2(\cX, \rho_X)$ denote the inclusion. 
Recall that the function minimizing the expected risk over the set of all measurable functions is the \emph{regression function}, given by 
\[ f_\rho(x) = \int_{\cY} y\;d\rho(y|x) \;,  \quad \rho_X- a.s. \;. \]
A solution  $f_\cH$ of \eqref{eq:min-prob} is given by the projection of $f_\rho$ onto the closure of  $Ran(\cS)$ in $L^2(\cX, \rho_X)$.

In Section \ref{sec-smoothness-promoting} we investigate the learning properties of an approximate SGD minimizer 
$\bar f_\bz$ based on i.i.d. data $\bz=(\bx, \by)=((x_i, y_i))_{i=1,...,n} \in (\cX \times \cY)^n$, 
under different regularity assumptions on $f_\cH$.  
Important for our analysis will be the covariance operator $\cT:=\cS^*\cS : \cH \to \cH$, given by 
\begin{equation}\label{eq:covariance}
  \cT   = \int_\cX  \inner{\cdot , K_x}_\cH K_x \; \rho_X(dx) \;,
\end{equation}  
with $K_x=(x, \cdot)$. 
Under Assumption \ref{ass:bounded}, the operator $\cT$ is positive and trace class, hence compact and satisfies $||\cT|| \leq trace(\cT)\leq \kappa^2 $. 
For more details, we refer to Appendix \ref{app-notation}. 

It is well known that smoothness of $f_\cH$ influences accuracy and convergence of any learning algorithm for solving \eqref{eq:min-prob}. 
In learning theory, the regularity of $f_\cH$ is measured by means of a  \emph{source condition}, see e.g.  
\cite{bauer2007regularization}, \cite{rastogi2017optimal}. 
In the most general version, it is assumed that there is an \emph{index function}, i.e. a non-decreasing continuous function $\varphi: (0, ||\cT||] \to \mbr_+$ 
such that $f_\cH=\varphi(\cT)h$, with \emph{source element} $h \in \cH$. That is, $f_\cH \in Ran(\varphi(\cT))$. 


The choice $\varphi(t)=t^r$ for $r\geq0$ 
is called \emph{H\"older source condition} and is of special interest.  Optimal learning rates for the excess risk for different algorithms 
such as regularized least squares \cite{optimalratesRLS}, \cite{fischer2017sobolev}, gradient 
descent or SGD  \cite{blanchard2018optimal}, \cite{DieuBa16}, \cite{pillaud2018statistical}, \cite{mucke2019beating} are of order $\cO(n^{-\frac{2r +1}{2(r+1)}})$ under such a 
condition. 
Consequently, the smoother $f_\cH$ (i.e. the larger $r$ is), the faster is the learning rate. These rates can further be 
enhanced to $\cO(n^{-\frac{2r +1}{2r+1+\nu}})$, for some $\nu \in (0,1]$, if the \emph{effective dimension} 
\begin{equation}
\label{eq:eff-dim-def}
\cN_{\cT}(\lam ):=trace\left( \cT(\cT + \lam)^{-1}\right) \;, \quad \lam >0\;,
\end{equation}
satisfies $\cN_{\cT}(\lam ) \lesssim \lam^{-\nu}$. This key quantity, a.k.a. \emph{capacity assumption} has been introduced in \cite{Zhang03} and has been 
applied since then in a variety of papers for deriving fast learning rates for regularization algorithms, e.g. \cite{optimalratesRLS, fischer2017sobolev, Lin18, blanchard2018optimal}. 

However, up to now, smoothness is only expressed as a source condition involving a function of the kernel covariance operator $\cT$, i.e. learning rates 
for more general smoothness assumptions for $f_\cH$ are completely missing. 
We fill this gap by analyzing the  regularization properties of SGD in a general Hilbert scale induced by an unbounded operator, 
see Section \ref{sec-learning-scales-theory},  
and generalize the results from  \cite{DieuBa16}, \cite{pillaud2018statistical}, \cite{mucke2019beating} for SGD where only H\"older conditions are considered.  
In particular, this approach allows now to also consider source conditions arising from other integral (or covariance) operators 
and thus having a broader applicability. This has not been 
analyzed for SGD before.




\section{Hilbert Scales: Theory and Examples}
\label{sec-learning-scales-theory}

As a preparatory  step we briefly review the theory of Hilbert scales from \cite{engl96}, see also \cite{krein1966scales}. 
To this end, we assume $\cH$ to be a Hilbert space  and we let $L$ be a densely defined, self-adjoint, linear, unbounded, and strictly positive operator 
$L: \cD(L) \subset \cH \to \cH$, that is, $\cD(L) = \cD(L^*)$ is dense in $\cH$ with  
\[ \inner{Lf, g}_{\cH} =  \inner{f, L^*g}_{\cH}\;, \quad f,g \in \cD(L) \;. \]
Note that strict positivity implies that $L^{-1}: \cH \to \cH$ exists as a bounded operator. By spectral theory, the 
operator $L^s: \cD(L^s) \to \cH$ is well-defined for any $s \in \mbr$. In particular, for $s\geq 0$ we define the Hilbert spaces 
\[ \cH_s := \cD(L^s)  \]
with inner product and norm 
\begin{equation}
\label{eq:def-norm}
\inner{f,g}_{s}:= \inner{L^sf, L^sg}_{\cH}\;, \quad ||f||_s:=||L^sf||_{\cH}  \;,
\end{equation}
for any $f,g \in \cD(L^s)$. 
We introduce furthermore the dual spaces $\cH_{-s}:= \cH_s^*$ 
and the sequence $(\cH_s)_{s \in \mbr}$ is called the \emph{Hilbert scale induced by $L$}. 
In particular, $\cH_0=\cH$ and for any $-\infty < s \leq t < \infty $ we have $\cH_t \hookrightarrow  \cH_s$ 
with dense and continuous embeddings. 

Hilbert scales have been introduced in the context of inverse problems where the operator $L$ is typically a differential operator and the 
spaces $(\cH_s)_{s}$ are e.g. Sobolev spaces. We provide some examples which are relevant for learning in RKHSs.

\begin{example}[{\bf Scale of general Reproducing Kernel Hilbert Spaces}]
{\rm 
Consider an RKHS $\cH = \cH_K$ with kernel $K: \cX \times \cX \to \mbr$. Let $L : \cD(L)\subset \cH_K \to \cH_K$ be as above. 
Since $L$ is strictly positive and $L^{-s}$ is bounded for any $s \geq 0$, the operator $L$ 
gives rise to a feature map $\phi_s: \cX \to \cH_s$ 
\[  \phi_s(x) := L^{-s}K_x \;, \quad K_x:= K(x, \cdot) \;, \quad x \in \cX \;, \]
with feature space $\cH_s$.  The kernel is given by 
\[  K_s(x,x'):=\inner{L^{-s}K_x , L^{-s}K_{x'}}_{s} \;, \]
with $K_0=K$. From \cite[Theorem 4.21]{StCh08} we know that there exists an associated RKHS $\cH_{K_s}$ and the map 
$V: \cH_{K_s} \to \cH_s$ with $Vf(x):= \inner{f,  \phi_s(x)}_s $ 
is a metric surjection. 
Note that we get a nested sequence $\cH_{K_t} \hookrightarrow  \cH_{K_s} \hookrightarrow  \cH_K$ 
of RKHSs, whenever $0\leq s \leq t$. 
}
\end{example}

\begin{example}[{\bf Covariance Scale}]
{\rm 
A particular instance arises from the choice $\cH_K$ an RKHS and $L=\cT^{-1}$, where $\cT: \cH_K \to \cH_K$ is the kernel covariance operator, 
defined as in \eqref{eq:covariance}. This gives $\cH_s=Ran(\cT^s)$ and these are the classical smoothness spaces used 
in the learning framework, see Section \ref{sec-learning-in-hilbert-scales}. We call this Hilbert scale the \emph{Covariance scale}. 
}
\end{example}

\begin{example}[{\bf Sobolev Spaces}]
\label{ex:sobolev}
{\rm 
Following \cite{de2019reproducing}, we let $\cM$ denote a $d$-dimensional Riemannian manifold, which is connected, complete and, 
with bounded geometry. Typical examples are the space $\cM=\mbr^d$ with the usual Riemannian structure induced by the Euclidean
inner product  or any compact connected submanifold of $\mbr^d$, e.g. $\cM=S^{d-1}=\{x \in \mbr^d \; : \; ||x||=1\}$. 
Letting $\cD'(\cM)$ denote the space of distributions on $\cM$, we consider the operator $L:=(Id+\Delta)^{1/2}: \cD'(\cM) \to L^2(\cM)$. 
Given $s \in [0,\infty)$ we define $H^s(\cM):=\cH_s$ to be space of distributions $f \in \cD'(\cM)$ such that there exists an $g \in L^2(\cM)$ satisfying 
\[ f= L^{-s}g = (Id+\Delta)^{-s/2} g\;, \quad ||f||_s =||g||_{L^2(\cM)} \;. \]
The spaces $H^s(\cM)$ become a Hilbert space w.r.t. the inner product $\inner{f_1, f_2}_s = \inner{L^sf_1 , L^s f_2}_{L^2(\cM)} $. 
Thus, the family $(H^s(\cM))_{s\geq 0}$ is part of a Hilbert scale generated by $L$ and are known as \emph{Sobolev spaces}. These are RKHSs provided 
the smoothness index satisfies $s>d/2$, see \cite[Theorem 8]{de2019reproducing}. 
}
\end{example}

\begin{example}[{\bf Diffusion spaces and Gaussian RKHSs}]
\label{ex:diffusion}
{\rm 
Let $\cM$ be as in Example \ref{ex:sobolev}. 
For all $s>0$ 
we denote by $e^{-\frac{s}{2}\Delta}$ the \emph{Heat kernel}, defined as bounded operator on $L^2(\cM)$ by spectral calculus, 
see \cite{reed2012methods}. We set  $L=e^{\frac{1}{2}\Delta}$ and 
\[ \cH_s = Ran( L^{-s}) = Ran(e^{-\frac{s}{2}\Delta}) \;. \]
The semi-group property\footnote{That means $e^{-s'\Delta} = e^{-s\Delta}e^{-(s'-s)\Delta}$. } 
of $e^{-s\Delta}$ (see \cite[Proposition 1]{de2019reproducing}) shows that $ \cH_s \hookrightarrow  \cH_{s'}$ for all $0<s' \leq s$ 
and thus the family $(\cH_s)_{s> 0}$ is part of a Hilbert scale generated by $L$. 
Moreover, as shown in \cite[Theorem 8]{de2019reproducing}, the 
spaces $\cH_s$ are RKHSs for any $s>0$,  called \emph{Diffusion spaces} and they satisfy $\cH_s \hookrightarrow H^{s'}(\cM)$ for any $s, s'>0$. 
In particular, if $\cM=\mbr^d$, the Heat kernel is explicitly given by 
\begin{equation}\label{eq:gauss}
 K_s(x,x') = (4\pi s)^{-d/2} e^{-\frac{1}{4s}||x-x'||^2_2} \;,
\end{equation} 
the so called (normalized) \emph{Gaussian kernel}, see \cite{folland1995introduction}. Note that the \emph{width} $s$ serves here as a smoothness parameter.  
}
\end{example}



\section{SGD in Hilbert Scales: Learning Rates}
\label{sec-smoothness-promoting}
\begin{figure}
  \centering
 \includegraphics[width=0.32\columnwidth, height=0.15\textheight]{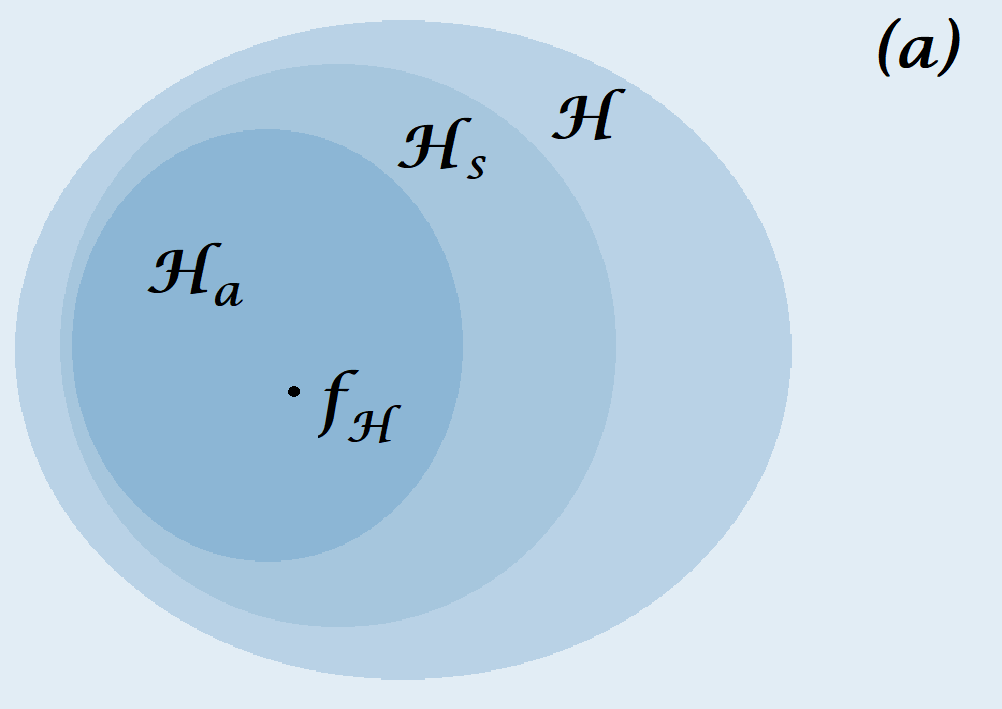}
 \includegraphics[width=0.32\columnwidth, height=0.15\textheight]{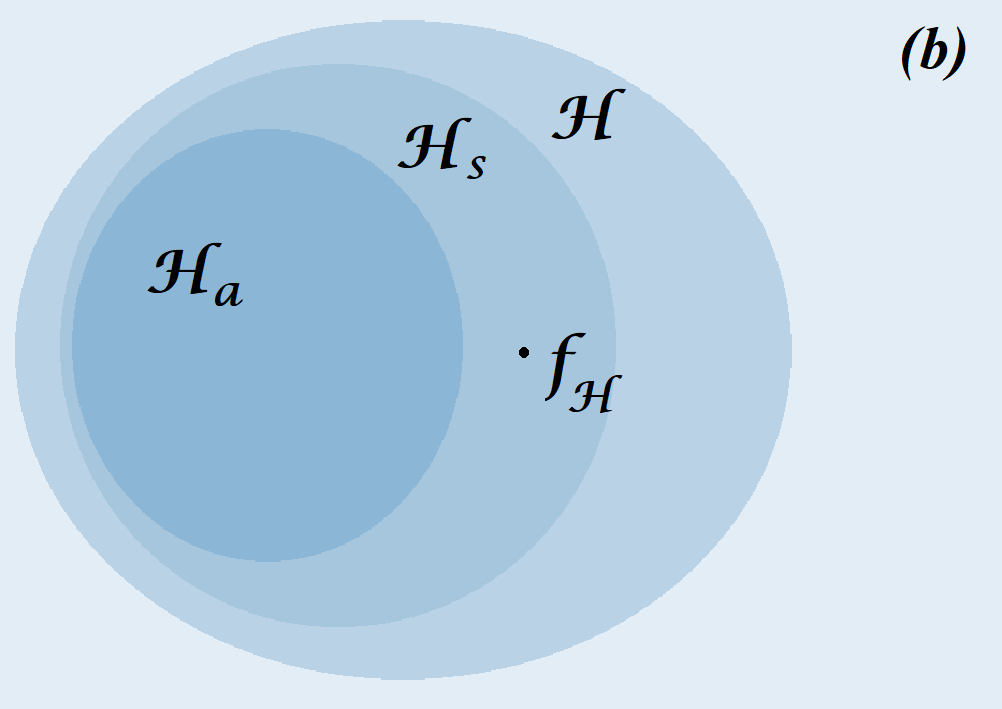}
 \includegraphics[width=0.32\columnwidth, height=0.15\textheight]{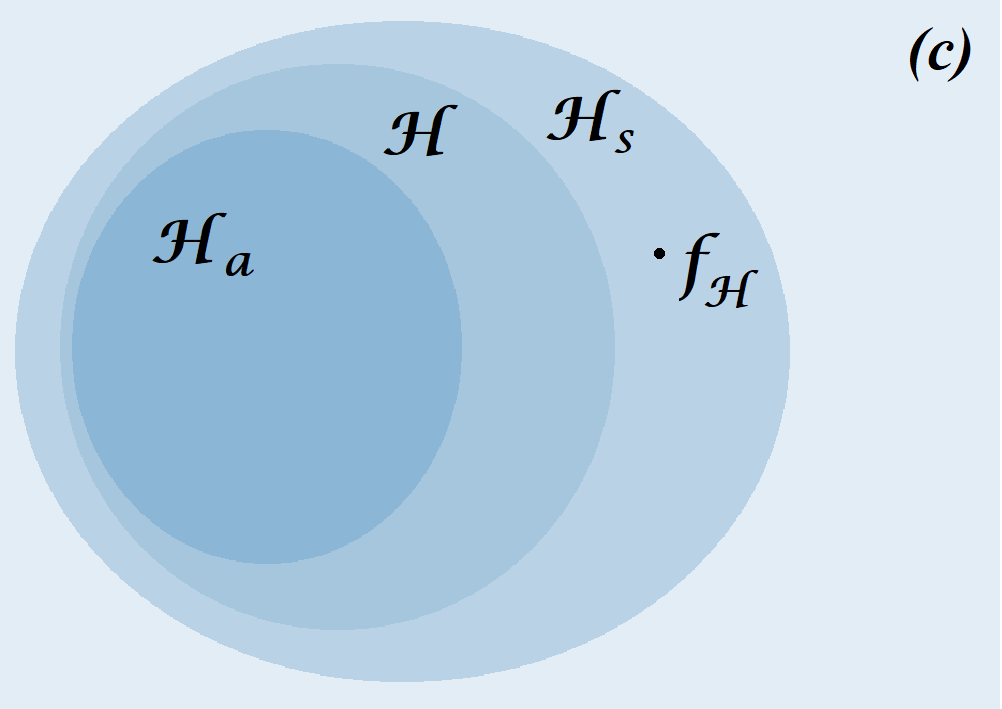} 
 \caption{{\small Illustration of the three different smoothness assumptions for the minimizer $f_\cH$. Regularization by tail-averaging SGD is performed in the space 
$\cH_s$. {\bf (a)} $0 \leq s \leq a$: The minimizer lies in $\cH_a$ and possesses the benchmark smoothness $a>0$. {\bf (b)} $0 \leq s < a$:  
The minimizer is less smooth and has a strictly 
positive distance $d_s$ to $\cH_a$. {\bf (c)} $-a \leq s \leq 0$: The minimizer does not belong to $\cH$ and is less regular. Note that $\cH \subseteq \cH_{s}$ in this case.}}
\label{fig:1}
\end{figure}


\subsection{Tail-averaged SGD in $\cH_s$}

We approximately solve the minimization 
problem (\ref{eq:min-prob}) by performing SGD in appropriate spaces $\cH_s$. 
Given $b \in [n]$, the mini-batch SGD recursion with constant step-size $\gamma >0$ in $\cH_s$ is given by $f_0=0$ and 
\begin{align}
\label{eq:SGD-recursion}
f_{t+1} = f_t - \gamma \frac{1}{b}\sum_{i=b(t-1)+1}^{bt} ( f_t(x_{j_i})  - y_{j_i})L^{-2s}K_{x_{j_i}}\;,
\end{align}
where $t=0,...,T$ and $j_1, ..., j_{bT}$ are i.i.d. random variables, distributed according to the uniform distribution
on $[n]$. Here, the number of passes over the data after $T$ iterations is $\ceil{bT/n}$. We are particularly interested in \emph{tail-averaging} 
with tail-length $T/2$, i.e. 
\begin{equation}
 \bar f_T := \frac{2}{T}\sum_{t=\floor{T/2} + 1}^T f_t \;.
\end{equation} 
The reason is twofold: In \cite{DieuBa16} it is shown that uniform averaging of the iterates leads to the possibility to 
choose constant step-sizes (instead of decaying $\gamma_t \sim t^{-\alpha}$, for some $\alpha \in  (0,1]$), making the algorithm more stable. However, as pointed out 
in \cite{mucke2019beating}, uniform  averaging suffers from \emph{saturation}, i.e. rates of convergence do not improve if the regularity of the 
objective function is large. Tail-averaging is known to lead to optimal rates of convergence also in the high smoothness regime.


\subsection{Smoothness Promoting SGD}
\label{sec-promoting}

Here, we consider the regular case where we assume that $f_\cH \in \cD(L^a)=\cH_a\subseteq \cH$ for some $a>0$, i.e. our model is well-specified and the target function  $f_\cH $ 
is smoother than just being in $\cH$. The operator $L^{-a}$ enforces smoothness of our  SGD iterates $\bar f_T$. 
For establishing the learning bounds in this case, we need:

\begin{assumption}[Link Condition]
\label{ass:link-condition3}
Assume  there exists $a>0$, $\underline{m}  >0$ such that for any $h \in \cH$
\begin{equation}
\label{eq:link2}
   \underline{m} ||h||_{-a} \leq  ||\cS h||_{L^2}  \;. 
\end{equation} 
\end{assumption}
Note that this assumption implies that $\cH_a \subseteq Ran(\cS)$. We are also interested in analyzing the special case where 
the benchmark smoothness $a>0$ is violated, that is $f_\cH \in \cH_s$, $0\leq s < a$, meaning that $f_\cH \not \in \cH_a$, see Figure \ref{fig:1}, $(b)$. 
For measuring the degree of violation of the benchmark, the concept of \emph{distance functions} is well suited, see e.g. \cite{flemming2011sharp}.

\begin{definition}[Distance Function]\label{def:distance1}
Given $a>0$ and  $0\leq s \leq a$ we define the \emph{distance function} $d_{s}: [0, \infty) \to [0, \infty)$ by 
\[ d_{s}(R):=\inf\{ ||L^s(f-f_\cH )||_\cH\;: \; f=L^{-a}h\;, ||h||_\cH \leq R \} \;. \]
\end{definition} 
The distance function is positive, decreasing, convex and continuous for all $0\leq R < \infty$ and tending to zero as $R\to \infty$, see \cite{hofmann2006approximate}. 
Hence, the 
minimizer exists and will henceforth be denoted as $f_{R} \in \cH_a$. 
Obviously, if $f_\cH \in \cH_a$, i.e. $f_\cH =L^{-a}h$ for some $h \in \cH$, then $f_\cH=f_{R_0} $, for some  $R_0 < \infty$ and $d_s(R_0)=0$. 
We now state our first main result, giving an upper bound for the excess risk. 

\begin{theorem}[Excess Risk]
\label{prop:newSGD}
Suppose Assumptions \ref{ass:bounded}, \ref{ass:link-condition3}, and  \ref{ass:moment} are satisfied. 
Let $\nu \in (0,1]$, $\beta=\frac{a-s}{2(a+s)}$ with $0\leq s\leq a$, $\gamma \kappa_s^2 < \frac{1}{4}$ and $R>0$. 
Assume further that $ Tr[\cT_s^{ \nu}] <\infty$ and that 
\begin{equation}
\label{eq:eff-dim-sample-size}
   n \geq \gamma T \max\{ 1, \cN_{\cT_s}(1/\gamma T) \}  \;.   
\end{equation}  
If $f_\cH \in \cD(L^s)$,  the excess risk satisfies\footnote{Throughout the manuscript we shortly write $A \lesssim  B$ for $A \leq c B$ for some $c \in \mbr_+$. } 
\begin{align} \label{eq:main1}
 \mbe\left[ \; ||\cS( \bar f_T - f_\cH )||^2_{L^2} \; \right] &\lesssim   (\gamma T)^{-1 } \left(\; d^2_{s}(R)   + R^2(\gamma T)^{-2\beta}   
      + M^2  \frac{\gamma T \cN_{\cT_s}(1/(\gamma T)) }{n}    \;\right)  \no \\
   & \quad   +  \frac{\Sigma^2_T (R)}{bT}  (\gamma T)^{\nu}    + \Omega_n(\gamma T)  \;, 
\end{align}
where the remainder $\Omega_n(\cdot)$ is of lower order and given by 
\[ \Omega_n(\gamma T) 
 = \frac{ || L^s\tilde g_T ||_\cH }{\gamma T} +   \sqrt{\frac{\gamma T}{n}}(d_s(R) + R^2(\gamma T)^{-\beta}) +  \gamma^2 T^5 \delta_n \]
with $\delta_n$ defined in \eqref{eq:delta}.  
Moreover, $\Sigma_T(R)$ obeys 
\[ \Sigma_T (R)  \lesssim  M + ||L^sf_\cH ||_\cH +  d_{s}(R)   + R(\gamma T)^{-\beta} \;.\]
\end{theorem}

The proof is given in Appendix \ref{app-proofs-SGD}. 
We draw now two conclusions, giving learning bounds in the two special well-specified cases. 
As in the classical approach, we assume a certain behavior for the effective dimension, see \eqref{eq:eff-dim-def}, of a slightly different covariance operator. 
More precisely, let $\cB_s:=\cS L^{-s}: \cH \to L^2(\cX, \rho_X)$ and $\cT_s:=\cB^*_s\cB_s: \cH \to \cH$. Note that $||\cT_s||\leq \kappa^2_s$, for some $\kappa_s>0$.

\begin{assumption}[Complexity]
\label{ass:effective-dim}
For some $\nu \in (0,1]$, the effective dimension obeys 
\[ \cN_{\cT_s}(1/\gamma T)  \lesssim (\gamma T)^\nu\;.  \]
\end{assumption}
Recall that for $f_\cH \in \cH_a$, see Figure \ref{fig:1}, $(a)$, 
the distance $d_s(R_0)$ appearing in \eqref{eq:main1} vanishes for some $R_0>0$. This gives:

\begin{corollary}
\label{cor:ratesSGD}
Let all Assumptions of Theorem \ref{prop:newSGD} be satisfied as well as Assumption \ref{ass:effective-dim}. 
Suppose that $f_\cH \in \cH_a$, for some $a>0$. 
Then, for any $n$ sufficiently large, the excess risk satisfies 
\[ \mbe\left[ ||\cS ( \bar f_{T_n} - f_\cH )||^2_{L^2} \right]   \lesssim R_0^2 \left( \frac{M^2}{R_0^2 n}  \right)^{\frac{1+2\beta}{1+2\beta + \nu }} \;, \]
for each of the following choices:\\
{\bf (a)} One pass SGD: $b=1$, $T_n \simeq n$, $\gamma_n\simeq  \frac{R_0^2}{M^2} \left(\frac{M^2}{R_0^2 n} \right)^{\frac{\nu +2\beta}{1+2\beta + \nu}}$.\\
{\bf (b)} Early Stopping and one pass SGD: 
$b_n\simeq n^{\frac{\nu+2\beta}{1+2\beta + \nu}}$, $\gamma = const.$, $T_n\simeq \left( \frac{R_0^2}{M^2}n\right)^{\frac{1}{1+2\beta + \nu}}$.\\
{\bf (c)}  Batch GD: $b=n$, $\gamma = const.$, $T_n\simeq \left( \frac{R_0^2}{M^2}n\right)^{\frac{1}{1+2\beta + \nu}}$. 
\end{corollary}

We turn now to the situation where for some $0\leq s <a$ the minimizer $f_\cH$ belongs to $\cH_s$, see Figure \ref{fig:1}, $(b)$. 
In this case, $d_s(R) >0$ and this will affect our learning bounds. 
In order to derive fast learning rates, we now need to determine the value of $R$, depending on $\gamma T$, 
such that the first two terms in \eqref{eq:main1} are balanced:    $d_{s}(R)     \asymp  R(\gamma T)^{-\beta}$.   
Thus, setting $\Gamma (R) :=R/ d_s(R)$, leads to 
\[ R =R_{\gamma T} = \Gamma^{-1}\left(  (\gamma T)^\beta \right) \;. \]
Actually, if the target satisfies for some 
$R_*>0$ a more \emph{general source condition}, see e.g. \cite{hofmann2007analysis}, \cite{rastogi2017optimal} or \cite{Lin18}, i.e.  
\begin{equation}\label{eq:source} 
f_\cH =\theta(L^{-1})h\;, \quad h \in \cH \;, \;\; ||h||_\cH \leq R_*\;,
\end{equation}
for some continuous increasing function $\theta: \mbr_+ \to \mbr_+$, 
we can give an upper bound for the distance function. 
Thanks to \cite[Theorem 5.9]{hofmann2007analysis}, see also \cite{rastogi2020inverse}, after rescaling, we obtain 
\begin{equation}
\label{eq:bound-distance}
 d_{s}(R) \leq  R \varphi\left(  \left(\frac{\varphi}{\tilde\theta}\right)^{-1}\left(\frac{R_*}{R}\right)  \right) \;,
\end{equation} 
with $\varphi(t)=t^{a-s}$, $\tilde\theta(t)=t^{-s}\theta(t)$. This finally gives 
$$ R_{\gamma T} \asymp (\gamma T)^{\frac{a\beta}{a-s}} \theta\left( (\gamma T)^{-\frac{\beta}{a-s}} \right) \;.$$
Note that the choice $\theta(t)=t^s$ gives $f_\cH \in \cH_s$, $d_s(R)\leq R_*$ and $R_{\gamma T}\asymp R_*(\gamma T)^\beta $.

\begin{corollary}
\label{cor:ratesSGD2}
Let all Assumptions of Theorem \ref{prop:newSGD} be satisfied as well as Assumption \ref{ass:effective-dim}. 
Suppose additionally that \eqref{eq:source} holds with $\theta(t)=t^s$.   
Then, for any $n$ sufficiently large, the excess risk satisfies 
\[ \mbe\left[ ||\cS ( \bar f_{T_n} - f_\cH )||^2_{L^2} \right]   \lesssim R_*^2 \left( \frac{M^2}{R_*^2  n}  \right)^{\frac{1}{1+\nu } }\;, \]
for each of the following choices: \\
{\bf  (a)} One pass SGD: $b=1$, $T_n \simeq n$, $\gamma_n\simeq  (R_*^2/M^2)\left(\frac{M^2}{R_*^2 n} \right)^{\frac{\nu}{1+\nu}}$. \\
{\bf (b)} Early Stopping and one pass SGD: $b_n\simeq n^{\frac{\nu}{1+\nu}}$, $\gamma = const.$, $T_n\simeq \left( \frac{R_*^2}{M^2}n\right)^{\frac{1}{1+\nu}}$. \\
{\bf (c)} Batch GD: $b=n$, $\gamma = const.$, $T_n\simeq \left( \frac{R_*^2}{M^2}n\right)^{\frac{1}{1+\nu}}$. 
\end{corollary}

We comment on the results derived above: 
{\bf (a)} Cor. \ref{cor:ratesSGD} and \ref{cor:ratesSGD2} precisely describe the interplay of $b, \gamma, T$. 
Comparing the results shows that violation of a benchmark smoothness in Hilbert scales slows down rates 
of convergence from $\cO(n^{-\frac{1+2\beta}{1+2\beta + \nu}})$ to $\cO(n^{-\frac{1}{1+v}})$ for various choices of $b, \gamma, T$. In other words, rates  
can be improved by SGD in Hilbert scales if the regularity of $f_\cH$ is sufficiently large. Moreover, comparing the stopping times for one pass SGD in both 
$(a), (b)$ reveals that a lack of regularity  leads to \emph{later} stopping, i.e. $T_n$ reduces from 
$\cO(n^{\frac{1}{1+2\beta + \nu}})$ to $\cO(n^{\frac{1}{1+v}})$. 

{\bf (b)} Comparing further $(b),(c)$ we observe that the setting
of $\gamma$ and  $T$ is the same and there is a full range of possible values for $b_n \in [n^{\frac{2\beta + \nu}{1+2\beta + \nu}}, n]$ where a const. stepsize 
is allowed, still ensuring optimality. As noted in \cite{LinRos17}, \cite{mucke2019beating} where the covariance scale $L=\cT^{-1}$ is considered, increasing the
minibatch size beyond a critical value does not yield any benefit. Compared to \cite{LinRos17}, with a critical batchsize of $\cO(n^{\frac{2\beta + 1}{1+2\beta + \nu}})$, we proved 
that tail-averaging can lead to a much smaller critical minibatch size also in Hilbert scales, and hence to more efficient computations, 
see also \cite{mucke2019beating}. 

{\bf (c)} In the special case where $L=\cT^{-1}$ and the minimizer belongs to $\cH_a$, then $f_\cH$ satisfies a classical H\"older source condition as described in Section 
\ref{sec-learning-in-hilbert-scales}. We exactly recover the known optimal bounds from previous works \cite{optimalratesRLS}, \cite{fischer2017sobolev} or \cite{blanchard2018optimal}. 
In fact, Cor. \ref{cor:ratesSGD2} reduces to a special case of \cite{mucke2019beating} where tail-averaging was shown to lead to optimal rates in well-specified models. 

{\bf (d)} Let us consider the Heat kernel $e^{-s\Delta}$ generating the Gaussian RKHSs from Example \ref{ex:diffusion} again. 
From the theory of Hilbert scales, the width $s$ can be viewed as a smoothness parameter. Then Cor. \ref{cor:ratesSGD} states that  the rate of convergence 
explicitly depends on the width appearing in the exponent of the learning rate\footnote{For Gaussian RKHSs and uniform distribution, 
the effective dimension behaves as $\cN_{\cT_s}(1/\gamma T)\lesssim \log(\gamma T)$ and we consider the worst case scenario where $\nu=1$.} 
in terms of $\beta=\frac{a-s}{2(a+s)}$.  This naturally occurs in the context of Hilbert scales. 

{\bf (e)  Optimality:} The bound in Cor. \ref{cor:ratesSGD2} is minimax optimal under the given Assumptions \ref{ass:effective-dim} and \eqref{eq:source} with $\theta(t)=t^s$ 
since $f_\cH \in \cH_s$, see \cite{optimalratesRLS}, \cite{blanchard2018optimal}. However, the learning rate provided in Cor. \ref{cor:ratesSGD} is known to be minimax optimal 
under the Assumptions \ref{ass:effective-dim} and $f_\cH \in \cH_a$  only if $0<a\leq 1/4$. Indeed, Lemma   \ref{lem:prelim3}  gives 
\[ ||L^{-a}h || \leq \underbar m ^{-\frac{1}{2}}||\cT_s^{\frac{1}{4}}h||_\cH =  \underbar m ^{-\frac{1}{2}}|| \cT_s^{a}\cT_s^{\frac{1}{4} -a}h||_\cH \;. \]
Thus, by Lemma \ref{lem:ranges}, $f_\cH =L^{-a}h $ implies $f_\cH = \cT_s^a \tilde h$ with $\tilde h = \cT_s^{\frac{1}{4} -a}h$, $||\tilde h||_\cH \leq ||\cT_s^{\frac{1}{4} -a}|| \; ||h||_\cH$ 
and $f_\cH$ satisfies a classical H\"older source condition in terms of the covariance operator $\cT_s$, 
ensuring optimality according to \cite{optimalratesRLS}, \cite{blanchard2018optimal}.
In \cite{rastogi2020inverse} in the context of inverse problems, the authors derive optimality under 
an additional \emph{lifting condition}, relating smoothness as given in terms of $L^{-1}$ to smoothness in terms of $\cT$. However, it is open to show optimality, 
i.e. to derive a matching lower bound of our bounds without this extra assumption and for $a>1/4$.   
 
{\bf (f) Tail-averaged Gradient Descent: } On our way proving error bounds for tail-averaged SGD we derive in Appendix \ref{app-error-analysis-smoothness-promoting} 
also error bounds for tail-averaged Gradient Descent under the same assumptions, being interesting in it's own right.


\subsection{Preconditioning SGD for mis-specified Models}
\label{subsec:Precond}

Now we draw our attention to the case where $f_\cH \not \in \cH$, as illustrated in Figure \ref{fig:1}, $(c)$. We show that tail-averaging SGD still allows to obtain 
fast learning rates under an appropriated assumption. This amounts to regularization in a Hilbert space with weaker norm by using $L^{-s}$ as a \emph{preconditioner} for 
suitable $s\leq 0$. Note this gives $\cH \subseteq \cH_{s}$.  

\begin{assumption}[Link Condition]
\label{ass:link-condition2}
\begin{enumerate}
\item
Assume  there exists $a>0$, $\bar m >0$ such that 
\begin{equation}
\label{eq:link1}
 ||\cS h||_{L^2} \leq \bar m ||h||_{-a} \;, \quad h \in \cH \;. 
\end{equation} 
\item 
We assume that $\cH_{-a} \hookrightarrow  L^2(\cX, \rho_X)$ and the 
extension of $\cS$ to $\cH_{-a}$ (again denoted by $\cS$) is continuous and injective. Moreover,  
for any $-a \leq s \leq 0$  one has $||\cB_s||\leq \kappa_s$ for some $\kappa_s < \infty$, where $\cB_s = \cS L^{-s}$. 
\end{enumerate}
\end{assumption}
Note that this assumption implies that for any $-a \leq s \leq 0$, the space $\cH_s$ is an RKHS, too. The kernel arises by means of the Riesz Representation Theorem, 
see \cite{reed2012methods}. Indeed, since the extension of $\cS$ to $\cH_{s}$ is continuous we have $\cS h(x) = \inner{h, \tilde K_x}_{\cH_s}$ for some $\tilde K_x \in \cH_s$ 
and for any $h \in \cH_s$. By Cauchy-Schwartz 
\[ |\cS h(x)| =  |\inner{h, \tilde K_x}_{\cH_s}| \leq ||h||_{\cH_s} ||\tilde K_x||_{\cH_s} \;. \]
Thus, the evaluation functionals on $\cH_s$ are continuous. 
We give an example where this condition is satisfied. 

\begin{example}[Switching between different Diffusion spaces]
\label{ex:pre}
{\rm 
Let $t>0$ and  $\cH:=Ran(e^{-\frac{t}{2}\Delta})$ be the diffusion space, defined in Example \ref{ex:diffusion}.  
We show in Appendix \ref{app-additional-results} that \eqref{eq:link1} is satisfied for any $0<a\leq t$, with $L=e^{\frac{1}{2}\Delta}$. 
Moreover, $L^{s}$ maps from $\cH$ to $\cH_s=Ran(L^{-s})$ for any $0\leq s\leq t$ and $\cB_{s}$ is bounded for any $-a \leq s \leq 0$. 
The operator $L^s$ acts as a \emph{preconditioner} if $f_\cH \in \cH_s \supset \cH$. 
}
\end{example}

\begin{example}[Switching between different orders of smoothness]
\label{ex:pre2}
{\rm 
Let $s\leq 0$, $t>0$ and  $\cH:=Ran(e^{-\frac{t}{2}\Delta})$ be the diffusion space, defined in Example \ref{ex:diffusion}. For $u\geq 0$ define 
\[ L:=e^{\frac{1}{2}\Delta}(Id + \Delta)^{-\frac{u}{2}}\;.\] 
We show in in Appendix \ref{app-additional-results} that $L^{-s}$ maps from the Gaussian RKHS $\cH$ into the larger Sobolev space $H^{-su}(\cM)$ 
(note $s\leq 0$), see Example \ref{ex:sobolev}, 
provided $s +t\leq 0$. Moreover,  Assumption \ref{ass:link-condition2} is satisfied for any $-a\leq s$ with $0< t\leq a$. 
}
\end{example}

For more examples we refer to Section \ref{app-additional-results}, in the context of inverse problems we refer to \cite{egger2005preconditioning}. 

For deriving our error bounds we assume an a-priori smoothness for $f_\cH$ in a shifted scale.

\begin{assumption}[Source Condition]
\label{ass:SC-pre}
Let $a, s$ and $\cB_s$ as in Assumption \ref{ass:link-condition2} and let $r>-a$. 
Assume  there exists $R >0$ and $h \in \cH$ satisfying $||h||_\cH \leq R$ such that 
\begin{equation}
\label{eq:SC-pre}
  f_\cH = L^{-s}\cT_s^{\beta}h\;, \quad \beta = \frac{r-s}{2(a+s)}. 
\end{equation} 
\end{assumption}
The link condition implies that $Ran( L^{-s}\cT_s^{\frac{r-s}{2(a+s)}}) \subseteq Ran(\cT_s^{\frac{r}{2(a+s)}})$ 
and thus $f_\cH \in \cH$ if $r \geq 0$ and $f_\cH \not \in \cH$ if $-a<r <0$, see Lemma \ref{lem:prelim}. Under 
the assumptions given above, we can also give learning bounds for preconditioning SGD for mis-specified models.

\begin{theorem}[Excess Risk]
\label{theo:preSGD}
Suppose Assumptions \ref{ass:bounded}, \ref{ass:link-condition3}, \ref{ass:SC-pre} and  \ref{ass:moment} are satisfied with $a>0$, $-a/2 \leq s \leq 0$.  
Let further $\nu \in (0,1]$, $\gamma \kappa_s^2 < \frac{1}{4}$, $ Tr[\cT_s^{ \nu}] <\infty$ and \eqref{eq:eff-dim-sample-size} be satisfied. 
Then the excess risk for preconditioned tail-averaged SGD satisfies \eqref{eq:main1} with $\beta=\frac{r-s}{2(a+s)}$. 
\end{theorem}

\begin{corollary}
\label{cor:pre-SGD}
Suppose all assumptions of Theorem \ref{theo:preSGD} and Assumption \ref{ass:effective-dim} are satisfied. Then the 
learning rate in \eqref{cor:ratesSGD} also holds with $\beta=\frac{r-s}{2(a+s)}$, $r\geq s\geq -a/2$, for any of the choices of $b_n, \gamma_n, T_n$ given in 
items (a), (b) and (c), provided that $\nu +2\beta >0$.
\end{corollary}

We highlight the most important consequence of our theory: {\bf  Preconditioning reduces the number of iterations. }  
Corollary \ref{cor:pre-SGD} $(b)$ reveals the benefit of preconditioning SGD. For one pass SGD, the number of iterations required is 
$T_n=\cO(n^{\frac{1}{1+\nu + 2\beta}})$ while one pass SGD in $\cH$ requires $T_n=\cO(n^{\frac{1}{1+\nu}})$ iterations, see \cite{mucke2019beating}. This    
is substantially more if $s\leq r$.


Minimax optimality of the rate in Cor. \ref{cor:pre-SGD} follows from the classical 
theory \cite{optimalratesRLS}, \cite{blanchard2018optimal} if $\beta = 0$, i.e. $r=s$. 

\section{Numerical Illustrations}
\label{sec-numerics}
In this section we give some empirical illustration supporting our theoretical findings. We concentrate on three different aspects of our results: 
The effect of preconditioning, the effect of mini-batching and the relation between stepsize and smoothness. All our experiments are 
conducted on synthetic data and follow the model $Y_j=f_\cH(X_j)+ \eps_j$, where $\eps_j \sim \cN(0,0.01)$ and $X_j \sim \cU[-1,1]$, $j=1,...,n$.

\begin{figure}[h]
\centering

 \includegraphics[width=0.35\columnwidth, height=0.17\textheight]{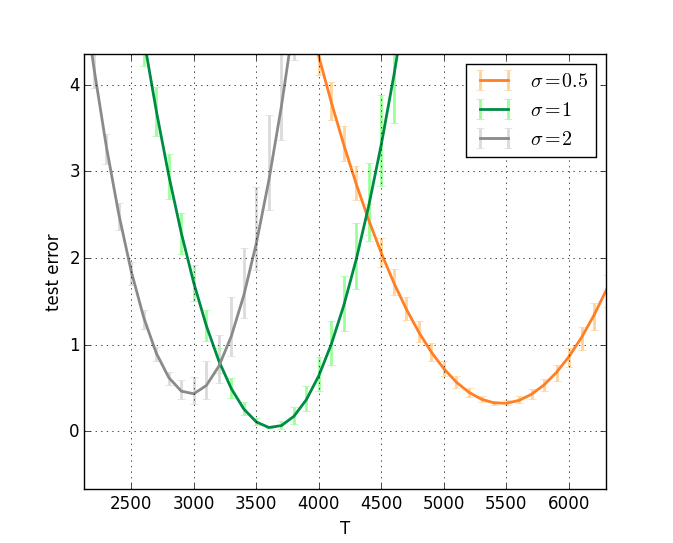}
 \includegraphics[width=0.34\columnwidth, height=0.17\textheight]{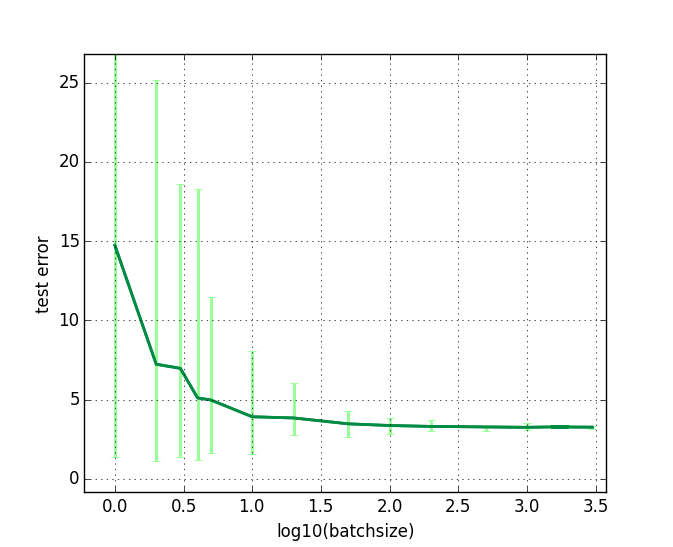}
 \caption{{\small Illustration of the effect of preconditioning and existence of a critical batchsize. Reported is the excess risk. 
{\bf Left:} Preconditioning reduces the number of iterations required to get to the minimum. 
{\bf Right:} Performance does not improve after a critical batchsize is reached. }}
\label{fig:2}
\end{figure}

\begin{figure}[h]
\centering
\includegraphics[width=0.34\columnwidth, height=0.17\textheight]{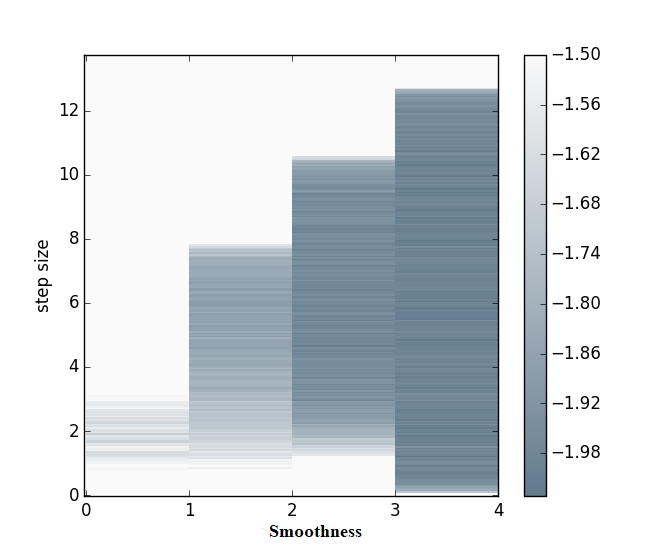}
\includegraphics[width=0.34\columnwidth, height=0.17\textheight]{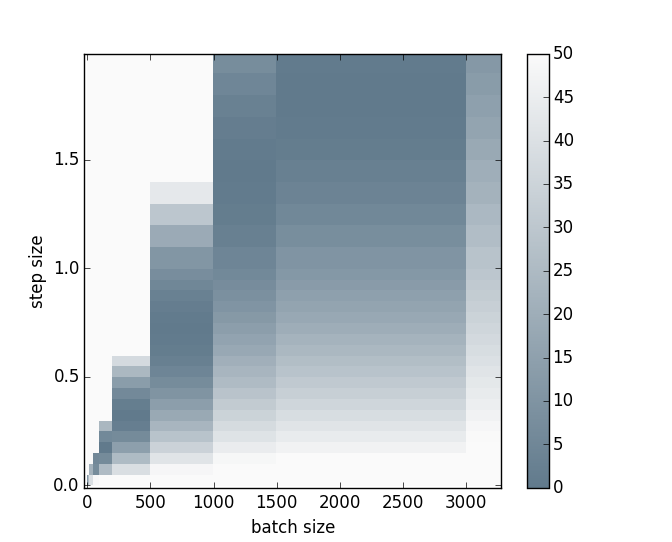}
\caption{{\small Illustration of the relation between stepsize, smoothness and batchsize. Reported is the excess risk. 
{\bf Left:} Stepsize increases with smoothness along the scale of Sobolev spaces for a smooth well-specified model.
{\bf Right:} Stepsize depends linearly on the batchsize. }}
\label{fig:3}
\end{figure}

{\bf (1) Preconditioning reduces the number of iterations required to get to the minimum: } 
Here, we pick up Example \ref{ex:pre} and let $f_\cH$ be a sum of Gaussians with width $s=1$. 
We take $n=3\cdot 10^3$, $b=300$ and $\gamma = 10^{-3}$. Plotted in the left plot in Fig. \ref{fig:2} are the test errors 
using three different Gaussians generating the Gaussian scale $\cH_{2} \subset \cH_1 \subset \cH_{0.5}$
with width $\sigma \in \{0.5, 1, 2\}$. The operator $L=e^{\frac{1}{2}\Delta}: \cH_2 \to \cH_1$ acts as a preconditioner, yielding less  
iterations as predicted by Cor. \ref{cor:pre-SGD}. On the other hand, we also see that a smoother model requires less iterations (see Cor. \ref{cor:ratesSGD} and \ref{cor:ratesSGD2}), that is, 
regularization in $\cH_{0.5}$ slows down convergence.  

{\bf (2) Existence of a critical batchsize: } 
The minimizer $f_\cH$ is chosen to be a linear combination of Mat\'ern kernels of order 3 and belongs therefore to the Sobolev space $H^{3+\frac{1}{2}}(\mbr)$.  
We perform SGD with constant stepsize $\gamma=10$, $T=3$ and $n=3\cdot 10^3$ in the Sobolev space $H^{3+\frac{1}{2}}(\mbr)$.  
Then our model is well-specified.  The batchsize ranges in $\{1,..., n\}$.  
The result is given in the second plot in Fig. \ref{fig:2}. As predicted by Cor. \ref{cor:ratesSGD}, \ref{cor:ratesSGD2} (see also remark (b) in the discussion), 
performance does not improve anymore after a critical batchsize is reached. 

{\bf (3) Stepsize increases with smoothness along the scale of Sobolev spaces for a smooth well-specified model: } The minimizer $f_\cH$ is a linear 
combination of the Mat\'ern Kernel of order $3$. We perform one pass SGD with $n=10^3$ in Sobolev spaces $H^{s+\frac{1}{2}}(\mbr)$, with $s=0,1,2,3$ with 
Mat\'ern Kernels of order $s$. Cor. \ref{cor:ratesSGD} predicts that with increasing degree of smoothness $s$, the stepsize $\gamma$ increases (equivalently, the 
number of iterations decreases). This is shown in the left plot in Fig. \ref{fig:3}. 

{\bf (4) Stepsize depends linearly on the batchsize: } 
The minimizer $f_\cH$ is chosen to be a sum of Gaussians as in Eq. \ref{eq:gauss}. 
We perform one pass SGD, i.e. $T = n/b$, with $n=3\cdot 10^3$ in the Sobolev space $H^{3+\frac{1}{2}}(\mbr)$ by utilizing the \emph{Mat\'ern} Kernel of order $3$.  
Then our model is well-specified. 
The result is given in the right plot in Fig. \ref{fig:3}. As predicted by Cor. \ref{cor:ratesSGD}, \ref{cor:ratesSGD2}, performance remains largely constant as 
$\gamma \cdot b$ remains constant.


\vspace{2cm}

\noindent
{\Large \bf Acknowledgments}
\\
\\
N.M. acknowledges funding by the Deutsche Forschungsgemeinschaft 
(DFG, German Research Foundation) under Germany's Excellence Strategy \emph{The Berlin Mathematics Research Center MATH+ } (EXC-2046/1, project ID: 390685689). 
\\
E.R. acknowledges funding by the \emph{Potsdam Graduate School}. 
\\
\\
N.M. is grateful for fruitful discussions with Markus Klein and helpful comments. 
N.M. is also thankful to Peter Math\'e for a nice conversation about regularization in Hilbert Scales. 


\bibliographystyle{plain}
\bibliography{bib_SGD}

\begin{thebibliography}{10}

\bibitem{aronszajn1950theory}
Nachman Aronszajn.
\newblock Theory of reproducing kernels.
\newblock {\em Transactions of the American mathematical society},
  68(3):337--404, 1950.

\bibitem{BacMou13}
Francis~R. Bach and Eric Moulines.
\newblock Non-strongly-convex smooth stochastic approximation with convergence
  rate o(1/n).
\newblock In {\em {NIPS}}, pages 773--781, 2013.

\bibitem{bauer2007regularization}
Frank Bauer, Sergei Pereverzev, and Lorenzo Rosasco.
\newblock On regularization algorithms in learning theory.
\newblock {\em Journal of complexity}, 23(1):52--72, 2007.

\bibitem{blanchard2019lepskii}
Gilles Blanchard, Peter Math{\'e}, and Nicole M{\"u}cke.
\newblock Lepskii principle in supervised learning.
\newblock {\em arXiv preprint arXiv:1905.10764}, 2019.

\bibitem{blanchard2018optimal}
Gilles Blanchard and Nicole M{\"u}cke.
\newblock Optimal rates for regularization of statistical inverse learning
  problems.
\newblock {\em Foundations of Computational Mathematics}, 18(4):971--1013,
  2018.

\bibitem{bottcher2006convergence}
Albrecht B{\"o}ttcher, Bernd Hofmann, Ulrich Tautenhahn, and Masahiro Yamamoto.
\newblock Convergence rates for tikhonov regularization from different kinds of
  smoothness conditions.
\newblock {\em Applicable Analysis}, 85(05):555--578, 2006.

\bibitem{optimalratesRLS}
Andrea Caponnetto and E.~De~Vito.
\newblock Optimal rates for regularized least-squares algorithm.
\newblock {\em Foundations of Computational Mathematics}, 7(3):331--368, 2006.

\bibitem{Cotter11}
Andrew Cotter, Ohad Shamir, Nati Srebro, and Karthik Sridharan.
\newblock Better mini-batch algorithms via accelerated gradient methods.
\newblock In J.~Shawe-Taylor, R.~S. Zemel, P.~L. Bartlett, F.~Pereira, and
  K.~Q. Weinberger, editors, {\em Advances in Neural Information Processing
  Systems 24}, pages 1647--1655. Curran Associates, Inc., 2011.

\bibitem{de2019reproducing}
Ernesto De~Vito, Nicole M{\"u}cke, and Lorenzo Rosasco.
\newblock Reproducing kernel hilbert spaces on manifolds: Sobolev and diffusion
  spaces.
\newblock {\em arXiv preprint arXiv:1905.10913}, 2019.

\bibitem{DieuBa16}
Aymeric Dieuleveut and Francis Bach.
\newblock Nonparametric stochastic approximation with large step-sizes.
\newblock {\em Ann. Statist.}, 44(4):1363--1399, 08 2016.

\bibitem{DieFlaBac17}
Aymeric Dieuleveut, Nicolas Flammarion, and Francis Bach.
\newblock Harder, better, faster, stronger convergence rates for least-squares
  regression.
\newblock {\em Journal of Machine Learning Research}, 18:101:1--101:51, 2017.

\bibitem{egger2005preconditioning}
Herbert Egger and Andreas Neubauer.
\newblock Preconditioning landweber iteration in hilbert scales.
\newblock {\em Numerische Mathematik}, 101(4):643--662, 2005.

\bibitem{engl96}
Heinz~Werner Engl, Martin Hanke, and Andreas Neubauer.
\newblock {\em Regularization of inverse problems}, volume 375.
\newblock Springer Science \& Business Media, 1996.

\bibitem{fischer2017sobolev}
Simon Fischer and Ingo Steinwart.
\newblock Sobolev norm learning rates for regularized least-squares algorithm.
\newblock {\em arXiv preprint arXiv:1702.07254v2}, 2019.

\bibitem{flemming2011sharp}
Jens Flemming, Bernd Hofmann, and Peter Math{\'e}.
\newblock Sharp converse results for the regularization error using distance
  functions.
\newblock {\em Inverse Problems}, 27(2):025006, 2011.

\bibitem{folland1995introduction}
Gerald~B Folland.
\newblock {\em Introduction to partial differential equations}.
\newblock Princeton university press, 1995.

\bibitem{HarRecSin16}
Moritz Hardt, Ben Recht, and Yoram Singer.
\newblock Train faster, generalize better: Stability of stochastic gradient
  descent.
\newblock In {\em {ICML}}, volume~48 of {\em {JMLR} Workshop and Conference
  Proceedings}, pages 1225--1234. JMLR.org, 2016.

\bibitem{HK14}
Elad Hazan and Satyen Kale.
\newblock Beyond the regret minimization barrier: optimal algorithms for
  stochastic strongly-convex optimization.
\newblock {\em The Journal of Machine Learning Research}, 15(1):2489--2512,
  2014.

\bibitem{hofmann2006approximate}
Bernd Hofmann.
\newblock Approximate source conditions in tikhonov--phillips regularization
  and consequences for inverse problems with multiplication operators.
\newblock {\em Mathematical Methods in the Applied Sciences}, 29(3):351--371,
  2006.

\bibitem{hofmann2007analysis}
Bernd Hofmann and Peter Math{\'e}.
\newblock Analysis of profile functions for general linear regularization
  methods.
\newblock {\em SIAM Journal on Numerical Analysis}, 45(3):1122--1141, 2007.

\bibitem{krein1966scales}
S.~G. Krein and Yu~I. Petunin.
\newblock Scales of banach spaces.
\newblock {\em Russian Mathematical Surveys}, 21:85--159, 1966.

\bibitem{LinCamRos16}
Junhong Lin, Raffaello Camoriano, and Lorenzo Rosasco.
\newblock Generalization properties and implicit regularization for multiple
  passes {SGM}.
\newblock {\em CoRR}, abs/1605.08375, 2016.

\bibitem{LinRos17}
Junhong Lin and Lorenzo Rosasco.
\newblock Optimal rates for multi-pass stochastic gradient methods.
\newblock {\em Journal of Machine Learning Research}, 18:97:1--97:47, 2017.

\bibitem{Lin18}
Junhong Lin, Alessandro Rudi, Lorenzo Rosasco, and Volkan Cevher.
\newblock Optimal rates for spectral algorithms with least-squares regression
  over hilbert spaces.
\newblock {\em Applied and Computational Harmonic Analysis}, 2018.

\bibitem{mair1994tikhonov}
Bernard~A Mair.
\newblock Tikhonov regularization for finitely and infinitely smoothing
  operators.
\newblock {\em SIAM Journal on Mathematical Analysis}, 25(1):135--147, 1994.

\bibitem{mathe2007error}
Peter Mathe and Ulrich Tautenhahn.
\newblock Error bounds for regularization methods in hilbert scales by using
  operator monotonicity.
\newblock {\em Far East Journal of Mathematical Sciences}, 24(1):1, 2007.

\bibitem{mucke2019beating}
Nicole M{\"u}cke, Gergely Neu, and Lorenzo Rosasco.
\newblock Beating sgd saturation with tail-averaging and minibatching.
\newblock In {\em Advances in Neural Information Processing Systems}, pages
  12568--12577, 2019.

\bibitem{nair2005regularization}
M~Thamban Nair, Sergei~V Pereverzev, and Ulrich Tautenhahn.
\newblock Regularization in hilbert scales under general smoothing conditions.
\newblock {\em Inverse Problems}, 21(6):1851, 2005.

\bibitem{natterer1984error}
Frank Natterer.
\newblock Error bounds for tikhonov regularization in hilbert scales.
\newblock {\em Applicable Analysis}, 18(1-2):29--37, 1984.

\bibitem{NemJudLan09}
Arkadi Nemirovski, Anatoli Juditsky, Guanghui Lan, and Alexander Shapiro.
\newblock Robust stochastic approximation approach to stochastic programming.
\newblock {\em SIAM Journal on Optimization}, 19(4):1574--1609, 2009.

\bibitem{orabona}
Francesco Orabona.
\newblock Simultaneous model selection and optimization through parameter-free
  stochastic learning.
\newblock In {\em Advances in Neural Information Processing Systems}, pages
  1116--1124, 2014.

\bibitem{pillaud2018statistical}
Loucas Pillaud-Vivien, Alessandro Rudi, and Francis Bach.
\newblock Statistical optimality of stochastic gradient descent on hard
  learning problems through multiple passes.
\newblock In {\em Advances in Neural Information Processing Systems}, pages
  8114--8124, 2018.

\bibitem{PolJu92}
B.~T. Polyak and A.~B. Juditsky.
\newblock Acceleration of stochastic approximation by averaging.
\newblock {\em SIAM J. Control Optim.}, 30(4):838--855, jul 1992.

\bibitem{RSS12}
Alexander Rakhlin, Ohad Shamir, and Karthik Sridharan.
\newblock Making gradient descent optimal for strongly convex stochastic
  optimization.
\newblock In {\em Proceedings of the 29th International Conference on Machine
  Learning (ICML)}, pages 1571--1578, 2012.

\bibitem{rastogi2020inverse}
Abhishake Rastogi and Peter Math{\'e}.
\newblock Inverse learning in hilbert scales.
\newblock {\em arXiv preprint arXiv:2002.10208}, 2020.

\bibitem{rastogi2017optimal}
Abhishake Rastogi and Sivananthan Sampath.
\newblock Optimal rates for the regularized learning algorithms under general
  source condition.
\newblock {\em Frontiers in Applied Mathematics and Statistics}, 3:3, 2017.

\bibitem{reed2012methods}
Michael Reed.
\newblock {\em Methods of modern mathematical physics: Functional analysis}.
\newblock Elsevier, 2012.

\bibitem{RosVil15}
Lorenzo Rosasco and Silvia Villa.
\newblock Learning with incremental iterative regularization.
\newblock In {\em {NIPS}}, pages 1630--1638, 2015.

\bibitem{R88}
David Ruppert.
\newblock Efficient estimations from a slowly convergent {Robbins--Monro}
  process.
\newblock Technical report, Cornell University Operations Research and
  Industrial Engineering, 1988.

\bibitem{SmaYao06}
Steve Smale and Yuan Yao.
\newblock Online learning algorithms.
\newblock {\em Foundations of Computational Mathematics}, 6(2):145--170, 2006.

\bibitem{StCh08}
I.~Steinwart and A.~Christmann.
\newblock {\em Support Vector Machines}.
\newblock Springer, 2008.

\bibitem{YaoTar14}
Pierre Tarres and Yuan Yao.
\newblock Online learning as stochastic approximation of regularization paths:
  Optimality and almost-sure convergence.
\newblock {\em {IEEE} Trans. Information Theory}, 60(9):5716--5735, 2014.

\bibitem{tautenhahn1996error}
Ulrich Tautenhahn.
\newblock Error estimates for regularization methods in hilbert scales.
\newblock {\em SIAM Journal on Numerical Analysis}, 33(6):2120--2130, 1996.

\bibitem{YinPon08}
Yiming Ying and Massimiliano Pontil.
\newblock Online gradient descent learning algorithms.
\newblock {\em Foundations of Computational Mathematics}, 8(5):561--596, 2008.

\bibitem{Zhang03}
T.~Zhang.
\newblock Effective dimension and generalization of kernel learning.
\newblock {\em Advances in Neural Information Processing Systems 2003}, 2003.

\end{thebibliography}






\newpage 
\appendix

{\bf \Large Appendix }

\section{Notation}
\label{app-notation}
To begin with our error analysis we introduce some further notation. 
We consider the two different cases where $s\geq 0$ and $s<0$.   
If $s \geq 0$ we know that under Assumption \ref{ass:bounded}  and by definition of $L: \cD(L)\subset \cH \to \cH$, the operator 
$\cB_s:=\cS L^{-s}:  \cH \to L^2(\cH, \rho_X)$ is bounded.  For $s<0$ we let Assumption \ref{ass:link-condition2}, $2.$, be satisfied. 
We introduce the adjoint $\cB_s^*:  L^2(\cH, \rho_X) \to \cH$, given by  
\[  \cB_s^*g = \int_\cX g(x)L^{-s}K_x\; d\rho_X(x)\;.  \]
The non-centered covariance operator $\cT_s:=\cB_s^*\cB_s : \cH \to \cH$ 
satisfies 
\[  \cT_s = \int_{\cX} \inner{\cdot , L^{-s}K_x }_\cH L^{-s}K_x \;d\rho_X(x)\;. \]
Finally, we define $\cL_s:=\cB_s\cB_s^*:  L^2(\cH, \rho_X) \to L^2(\cH, \rho_X) $. 
Under the given assumptions, the operators $\cT_s$ and $\cL_s$ are positive and trace class 
and they satisfy 
\[ ||\cT_s|| \leq trace(\cT_s) = \int_{\cX} ||L^{-s}K_x||_\cH^2 \;d\rho_X(x) \leq  \kappa_s^2  \] 
as well as $trace(\cT_s)=trace(\cL_s)$. 
In particular, we denote $\cB_0=\cS$, $\cT_0=\cT$ and $\cL_0=\cL$. 
Moreover, we have the isometry property 
\begin{equation}
\label{eq:iso1}
 ||\cS f||^2_{L^2} = \inner{\cS f, \cS f}_{L^2} = \inner{\cS^*\cS f,  f}_{\cH} = \inner{\sqrt{\cT} f, \sqrt{\cT} f}_{\cH} = ||\sqrt{\cT}f||^2_\cH \;,
\end{equation} 
and more generally, 
\begin{equation}
\label{eq:iso2}
   ||\cS f||_{L^2}  = || \cB_s L^s f ||_{L^2} =  ||\sqrt{\cT_s} L^s   f||_\cH   \,, 
\end{equation}    
for all $f \in \cH$.


In addition, we introduce the empirical counterparts of the above operators. If 
\[ (\bx, \by)=((x_1,y_1), ..., (x_n, y_n)) \in (\cX \times \cY)^n \]
is an i.i.d. sample from $\rho$, we denote: 
\begin{align*}
 \cS_{\bx,s} &: \cH \to \mbr^n \;, \quad (\cS_{\bx,s} f)_j = \inner{f,L^{-s} K_{x_j}}_\cH \;, \quad j=1,...,n \;, \\
 \cS^*_{\bx ,s}&: \mbr^n \to \cH\;, \quad  \cS^*_{\bx,s} \by = \frac{1}{n}\sum_{j=1}^n y_j L^{-s}K_{x_j} \;, \\
\cT_{\bx,s} &: \cH \to \cH\;, \quad \cT_{\bx,s}  = \frac{1}{n}\sum_{j=1}^n   \inner{\cdot, L^{-s}K_{x_j}}_\cH L^{-s}K_{x_j} \;, \\ 
 \cL_{\bx,s} &: \mbr^n \to \mbr^n  \;, \quad  \cL_{\bx,s} = \left(\frac 1 n \inner{L^{-s}K_{x_i}, L^{-s} K_{x_j}}_\cH \right)_{i,j=1,...,n} \;.
\end{align*}


\section{Calculations in  Hilbert Scales}
\label{app-HS}
In this section we collect some preparatory results related to Hilbert scales. 
\begin{lemma}
\label{lem:adjoints}
Let $s \geq 0$ and denote by $\cS_s=(\cS_{|_{\cH_s}})$ the restriction of $\cS$ to $\cH_s \subseteq \cH$.  
The adjoint operator $\cS^*_s: L^2(\cH, \rho_X) \to \cH_s $ is given by $\cS^*_s = L^{-2s}\cS^*$.
\end{lemma}

\begin{proof}[Proof of Lemma \ref{lem:adjoints}]
Let $h \in \cH_s$. Then $\cS_s h = \cS h$ and since $L^s$ is self-adjoint, we may write 
\begin{align*}
\inner{\cS_s h , g}_{L^2}&=\inner{\cS h , g}_{L^2} = \inner{ h , \cS^* g}_{\cH} \\
&= \inner{L^{-s} L^{s} h ,L^{s} L^{-s} \cS^* g}_{\cH} =  \inner{L^{s} h ,L^{s} L^{-2s} \cS^* g}_{\cH} \\
&=\inner{ h , L^{-2s} \cS^* g}_{\cH_s}  \;. 
\end{align*}
\end{proof}

\begin{lemma}[Heinz Inequality \cite{engl96}, Prop.8.21]
\label{lem:heinz}
Let $A$ and $B$ be two densely defined unbounded selfadjoint strictly positive operators on $\cH$ with $\cD(A)\subset \cD(B)$ and 
\[ ||Bh|| \leq ||Ah||\;, \quad \mbox{for all } h \in \cD(A) \;. \]
Then, for any $\nu \in [0,1]$ we have $\cD(A^\nu )\subset \cD(B^\nu )$ and 
\[ ||B^\nu h|| \leq ||A^\nu h||\;, \quad \mbox{for all } h \in \cD(A^\nu ) \;. \]
\end{lemma}


\[\]

\begin{lemma}[\cite{bottcher2006convergence}, Proposition 2.1]
\label{lem:ranges}
Let $S$ and $T$ be self-adjoint bounded operators on $\cH$ and suppose that $T$ is injective. If $||Sf|| \leq C ||Tf||$ for any $f \in \cH$ and 
for some $C<\infty$, then $Ran(S) \subset Ran(T)$ and $||T^{-1}S||\leq C$. 
\end{lemma}

\[\]

\begin{lemma}
\label{lem:prelim3}
Let Assumption \ref{ass:link-condition3} be satisfied. 
\begin{enumerate}
\item
For any $\nu \in [0,1]$ it holds 
\begin{align}
\cH_{\nu( a+s)} &\subset  Ran(\cT_s^{\nu/2}) = \cD(\cT_s^{-\nu/2})\;, \no \\
\label{eq:111}
||\cT_s^{\nu/2}f||_\cH &\geq \underbar m^\nu ||f||_{-\nu( a+s)} \;, \quad f \in \cH \;,   \\
\label{eq:11}
||\cT_s^{-\nu/2}f||_\cH &\leq \underbar m^{-\nu} ||f||_{\nu( a+s)} \;, \quad f \in \cH_{\nu( a + s)} \;.
\end{align}
\item 
For any $a >0$, $s\geq 0$ we have 
\[ ||L^{-s} f||_\cH \leq  \underbar m ^{-\frac{s}{a+s}}||\cT_s^{\frac{s}{2(s+a)}}f||_\cH \;. \]
\item If $0 \leq s \leq a$ we have 
\[  ||L^{-(a-s)}f||_\cH \leq   \underbar m^{-\frac{a-s}{a+s}} ||\cT_s^{\beta}f||_\cH    \]
and therefore 
\[ Ran(L^{-(a-s)}) \subseteq Ran(\cT_s^\beta ) \;, \quad \beta = \frac{a-s}{2(a+s)}\leq \frac{1}{2}\;. \]
\end{enumerate}
\end{lemma}

\begin{proof}[Proof of Lemma \ref{lem:prelim3}]
\begin{enumerate}
\item The first part follows from 
\cite[Proposition]{egger2005preconditioning}.
\item The second assertion follows from \eqref{eq:11} by applying the Heinz Inequality, Lemma \ref{lem:heinz}. 
\item The third part follows from \eqref{eq:111} by applying the Heinz Inequality, Lemma \ref{lem:heinz} and Lemma \ref{lem:ranges}. 
\end{enumerate}
\end{proof}

\[\]

\begin{lemma}
\label{lem:prelim}
Let Assumption \ref{ass:link-condition2} be satisfied. 
\begin{enumerate}
\item For any $\nu \in [0,1]$ we have 
\begin{align}
\cD(\cT_s^{-\nu/2}) &= Ran(\cT_s^{\nu/2}) \subset \cH_{\nu(a+s)} \;, \no \\
\label{eq:110}
||\cT_s^{\nu/2}f||_\cH &\leq \bar m^\nu ||f||_{-\nu(a+s)} \;, \quad f \in \cH \;, \\
\label{eq:10}
||\cT_s^{-\nu/2}f||_\cH &\geq \bar m^{-\nu} ||f||_{\nu(a+s)} \;, \quad f \in \cD(\cT_s^{-\nu/2}) \;.
\end{align}
\item For any $f \in \cH$ and $-a/2 \leq s \leq 0$ we have 
\[  ||L^{-s}f||_\cH \leq \bar m ^{-\frac{s}{a+s}}||\cT_s^{\frac{s}{2(s+a)}}f||_\cH \;. \]
\item Assume that for $-a/2 \leq s \leq 0$ and $r \geq 0$ 
\[ f = L^{-s}\cT_s^{\beta}h\;, \quad \beta = \frac{r-s}{2(a+s)}\;,\]
for some $h \in \cH$. Then $f \in \cH$. 
\end{enumerate}
\end{lemma}

\begin{proof}[Proof of Lemma \ref{lem:prelim}]
\begin{enumerate}
\item
The first part follows from 
\cite[Proposition]{egger2005preconditioning}. 
\item The second assertion follows from \eqref{eq:10} by applying the Heinz Inequality, Lemma \ref{lem:heinz}.  
\item The third part follows from $(ii)$. Indeed, we have for any $h \in \cH$
\begin{align*}
||L^{-s} \cT_s^\beta h||_\cH &\leq  \bar m ^{-\frac{s}{a+s}}||\cT_s^{\frac{s}{2(s+a)}} \cT_s^\beta h||_\cH \\
&= \bar m ^{-\frac{s}{a+s}}||\cT_s^{\frac{r}{2(s+a)}} h||_\cH \;.
\end{align*}
Thus, thanks to  Lemma \ref{lem:ranges} we obtain $Ran(L^{-s} \cT_s^\beta ) \subseteq  Ran(\cT_s^{\frac{r}{2(s+a)}})$. The claim follows since 
$Ran(\cT_s^\alpha)\subseteq \cH$ if $\alpha \geq 0$.
\end{enumerate}
\end{proof}


\section{Error Bounds Tail-Averaged Gradient Descent in Hilbert Scales}
\label{app-error-analysis-smoothness-promoting}
In this section we derive learning rates for tail-averaged Gradient Descent. To this end, we consider the 
GD recursion in $\cH_s$ with constant stepsize $\gamma >0$ as given by 
\begin{equation}
\label{eq:barGD0}
g_{t+1} =  g_t - \gamma \frac{1}{n}\sum_{j=1}^n( g_t(x_j)-y_j)L^{-2s}K_{x_j}\;, 
\end{equation} 
where $t=0,...,T$. The tail-averaged updates are then defined by 
\begin{equation}
\label{eq:barGD}
 \bar g_T := \frac{2}{T}\sum_{t=\floor{T/2}+1}^T  g_t \;.  
\end{equation} 
Repeating the arguments in  \cite{mucke2019beating} shows that $\bar g_T$ can be written in closed form as 
\[ \bar g_T  =  L^{-s} \bar G_T(\cT_{\bx , s}) \cB^*_{\bx , s} \bar y \;,  \]
for a filter function $\bar G_T: (0, \kappa_s^2] \to \mbr$ defined on the spectrum $(0, \kappa_s^2]$ of $\cT_s$, 
see e.g. \cite{bauer2007regularization, engl96, blanchard2018optimal}. 
More precisely, 
\[ \bar G_T(\sigma) = \frac{2}{T}\sum_{t=\floor{T/2}+1}^T  G_t(\sigma ) \;, \]
with $G_t: (0, \kappa_s^2] \to \mbr$ given by 
\[ G_t(\sigma)=\gamma \sum_{k=0}^{t-1}(1-\gamma \sigma)^k \;,\]
being the GD filter function. We further introduce the \emph{residual} $\bar R_T: (0, \kappa_s^2] \to \mbr$ by 
\[ \bar R_T(\sigma )=1-\sigma \bar G_T(\sigma )  \;.\]
For the filter function properties of $\bar G_T$ and $\bar R_T$ we refer to  \cite{mucke2019beating}.

\subsection{Smoothness Promoting GD}

Let $s \geq 0$. In what follows we bound 
$||\cT_s^\alpha L^s (\bar g_T  - f_\cH)||_\cH $ for any $\alpha \in [0,1/2]$ and with $f_\cH \in \cD(L^s)$. Remember, 
according to \eqref{eq:iso2}, we have  
\[  ||\sqrt{\cT_s} L^s (\bar g_T  - f_\cH)||_\cH = || \cS( \bar g_T - f_\cH )||_{L^2}\;, \] 
while choosing $\alpha =0$ gives 
\[ || L^s (\bar g_T  - f_\cH)||_\cH = || \bar g_T - f_\cH ||_{s}\;. \]
We are also interested in deriving the error bound in $\cH$-norm. If the link condition in Assumption \ref{ass:link-condition3} is satisfied, then 
Lemma \ref{lem:prelim3} gives us with $\beta = \frac{a-s}{2(a+s)}\leq \frac{1}{2}$
\begin{align*}
|| \bar g_T - f_\cH  ||_\cH &= || L^{-s}L^s(\bar g_T - f_\cH)  ||_\cH \\
&\leq  \underbar m^{-\frac{a-s}{a+s}} ||\cT_s^{\beta}L^s(\bar g_T - f_\cH)||_\cH \;.
\end{align*}

{\bf Error Decomposition.}

Define 
\begin{align}
\label{eq:tildeu}
 \tilde g_T &:= L^{-s}\bar G_T(\cT_s) \cB^*_s \cS f_\cH \no \\
&= L^{-s}\bar G_T(\cT_s)\cT_s L^s f_\cH   
\end{align}  
and write 
\begin{align}
\label{eq:classical}
  \bar g_T -   f_\cH  &= \underbrace{(\bar g_T -  \tilde g_T)}_{Estimation \;\; error} + \underbrace{(  \tilde g_T -   f_\cH )}_{Approximation \;\; error} \;. 
\end{align} 
Recall that $f_R \in \cH_a$ denotes the minimizer of the distance function $d_s$, defined in \eqref{def:distance1}. Then we obtain 
\begin{align}
\label{eq:I}
\bar g_T -  \tilde g_T &= L^{-s}\bar G_T(\cT_{\bx , s}) \cB^*_{\bx , s} \bar y  -  \tilde g_T  \no \\
&= L^{-s}\bar G_T(\cT_{\bx , s})(\cB^*_{\bx , s} \bar y - \cT_{\bx , s} L^s\tilde g_T ) + 
   L^{-s}\bar R_T(\cT_{\bx , s})L^s\tilde g_T \no \\
   &= L^{-s}\bar G_T(\cT_{\bx , s})\left( (\cB^*_{\bx , s} \bar y -\cB_s^*f_\cH )- (   \cT_{\bx , s} L^s\tilde g_T - \cT_s L^s \tilde g_T ) \right)  \no \\
  & \; \;\; +  L^{-s}\bar G_T(\cT_{\bx , s})\left(  \cB_s^*f_\cH - \cT_s L^s \tilde g_T \right) + L^{-s}\bar R_T(\cT_{\bx , s}) L^s\tilde g_T \;,
\end{align}
and 
\begin{align}
\label{eq:II}
\tilde g_T -   f_\cH &= L^{-s}\bar G_T(\cT_s)\cT_s L^s f_\cH - L^{-s}L^s f_\cH   \no \\
&= L^{-s}  \bar R_T(\cT_s)  L^s f_\cH   \no \\ 
&= L^{-s}  \bar R_T(\cT_s)  L^s( f_\cH - f_{R})+ L^{-s}  \bar R_T(\cT_s)  L^s f_{R} \;. 
\end{align}


\subsubsection{Bounding the Approximation Error}

\begin{proposition}\label{prop:II}
Suppose Assumptions \ref{ass:bounded}, \ref{ass:link-condition3} are satisfied and let $R>0$, $\beta=\frac{a-s}{2(a+s)}$, $\alpha \in [0, \frac{1}{2}]$. Then 
\[  ||\cT_s^\alpha L^s(\tilde g_T -   f_\cH )||_\cH \leq c_{a,s}(\gamma T)^{-\alpha}  \left( d_{s}(R)   + R(\gamma T)^{-\beta}\right) \;, \]
for some $c_{a,s} < \infty$. Moreover,  with $\lam = (\gamma T)^{-1}$, we also have 
\[  ||(\cT_s + \lam )^\alpha L^s(\tilde g_T -   f_\cH )||_\cH \leq c'_{a,s}(\gamma T)^{-\alpha}  \left( d_{s}(R)   + R(\gamma T)^{-\beta}\right) \;, \]
for some $c'_{a,s} < \infty$.
\end{proposition}

\begin{proof}[Proof of Proposition \ref{prop:II}]
Recall that $f_{R}$ denotes the element in $\cH_a$ realizing the minimum of the distance function in Definition \ref{def:distance1}. 
Then $L^sf_{R} = L^{s-a}h$, for some $h \in \cH$. Applying Lemma \ref{lem:prelim3} gives 
$L^{-(a-s)}h = \cT_s^{\beta}\tilde h$, for some $\tilde h \in \cH$, satisfying $||\tilde h|| \leq R$ and 
with $\beta=\frac{a-s}{2(a+s)}$. By \eqref{eq:II} 
and \cite[Lemma 3]{mucke2019beating} with $K=3$ we thus obtain 
\begin{align*}
 ||\cT_s^\alpha L^s(\tilde g_T -   f_\cH )||_\cH &\leq   ||\cT_s^\alpha \bar R_T(\cT_s)  L^s( f_\cH - f_{R})||_\cH  
                        +  ||\cT_s^\alpha\bar R_T(\cT_s)  L^s f_{R} ||_\cH   \no \\
  &\leq     d_{s}(R)||\cT_s^\alpha \bar R_T(\cT_s) ||            + R    ||\cT_s^{\alpha+ \beta} \bar R_T(\cT_s)  ||  \no \\
&\leq   c_{a,s}(\gamma T)^{-\alpha}  \left( d_{s}(R)   + R(\gamma T)^{-\beta}\right) \;,                                   
\end{align*}
for some $c_{a,s} < \infty$.

Moreover, by the same reasoning, 
\begin{align*}
 ||(\cT_s + \lam )^\alpha L^s(\tilde g_T -   f_\cH )||_\cH &\leq   ||(\cT_s + \lam )^\alpha \bar R_T(\cT_s)  L^s( f_\cH - f_{R})||_\cH  
                        +  ||(\cT_s + \lam )^\alpha \bar R_T(\cT_s)  L^s w_{R} ||_\cH   \no \\
  &\leq     d_{s}(R)||(\cT_s + \lam )^\alpha \bar R_T(\cT_s) ||            + R    ||(\cT_s + \lam)^{\alpha + \beta} \bar R_T(\cT_s)  ||  \no \\
&\leq   c'_{a,s}(\gamma T)^{-\alpha}  \left( d_{s}(R)   + R(\gamma T)^{-\beta}\right) \;,                                   
\end{align*}
for some $c_{a,s} < \infty$.
\end{proof}


\subsubsection{Bounding the Estimation Error}

\begin{proposition}\label{prop:III}
Suppose Assumptions \ref{ass:bounded}, \ref{ass:link-condition3} are satisfied. Let $\alpha \in [0, \frac{1}{2}]$. 
\begin{enumerate}
\item
We have  $\rho_X^{\otimes n}$-almost surely
\[  ||\cT_s^\alpha \bar R_T(\cT_{\bx , s}) L^s\tilde g_T ||_\cH \leq c_\alpha (\gamma T)^{-\alpha} \;\Xi ^\alpha_{\bx,s}(\lambda )  \; || L^s\tilde g_T ||_\cH \;, \]
for some $ c_\alpha < \infty$ and where $\Xi ^\alpha_{\bx,s}(\cdot )$ is defined in  Proposition \ref{prop:sub-proba}.
\item 
If additionally $\gamma$ and $ T$ satisfy the 
condition
\[  n \geq \gamma T \max\{1, \cN_{\cT_s}((\gamma T)^{-1})\}    \;,  \]
we have for any $\delta \in (0,1]$ with    probability $\rho^{\otimes n}$ at least $1-\delta$
\[  ||\cT_s^\alpha \bar R_T(\cT_{\bx , s}) L^s\tilde g_T ||_\cH \leq c'_{\alpha, \kappa_s} 
\log^{2\alpha}(8/\delta)(\gamma T)^{-\alpha}  \; || L^s\tilde g_T ||_\cH \;, \]
for some $ c'_{\alpha, \kappa_s} < \infty$. 
\end{enumerate}
\end{proposition}

\begin{proof}[Proof of Proposition \ref{prop:III}]
\begin{enumerate}
\item
By \cite[Lemma 2]{mucke2019beating}, \cite[Lemma 3]{mucke2019beating}, and since $\alpha \in (0,1/2]$ we find 
with $\lambda = 1/(\gamma T)$
\begin{align*}
&||\cT_s^\alpha \bar R_T(\cT_{\bx , s}) L^s\tilde g_T ||_\cH  \no \\
&\leq ||(\cT_s+ \lambda)^\alpha (\cT_{\bx , s} + \lambda )^{-\alpha} || \;
 ||(\cT_{\bx , s} + \lam)^{\alpha} \bar R_T(\cT_{\bx , s})||\; || L^s\tilde g_T ||_\cH  \no \\
&\leq  c_\alpha (\gamma T)^{-\alpha} \; \Xi ^\alpha_{\bx,s}(\lambda ) \; \; || L^s\tilde g_T ||_\cH \;,
\end{align*}
for some $ c_\alpha < \infty$.
\item 
The second assertion follows immediately from Proposition \ref{prop:sub-proba}. 
\end{enumerate}
\end{proof}

\[\]

\begin{proposition}\label{prop:IIII}
Suppose Assumptions \ref{ass:bounded}, \ref{ass:link-condition3} are satisfied and let $R>0$,  $\beta=\frac{a-s}{2(a+s)}$ and $\alpha \in [0, \frac{1}{2}]$. 
\begin{enumerate}
\item
We have  $\rho_X^{\otimes n}$-almost surely 
\[ ||\cT_s^\alpha \bar G_T(\cT_{\bx , s})\left(  \cB_s^*\cS f_\cH - \cT_s L^s \tilde g_T \right) ||_\cH \leq  
c_{a,s, \alpha }\Xi_{\bx,s}(\lambda )(\gamma T)^{-\alpha}  \left( d_{s}(R)   + R(\gamma T)^{-\beta}\right)\;, \]
for some  $c_{a,s, \alpha }<\infty$ and where $\Xi_{\bx,s}(\cdot )$ is defined in Proposition \ref{prop:sub-proba}. 
\item
If additionally $\gamma$ and $ T$ satisfy the 
condition
\[ n \geq \gamma T \max\{1, \cN_{\cT_s}((\gamma T)^{-1})\}    \;,  \]
we have for any $\delta \in (0,1]$ with    probability $\rho^{\otimes n}$ at least $1-\delta$
\[ ||\cT_s^\alpha \bar G_T(\cT_{\bx , s})\left(  \cB_s^*\cS f_\cH - \cT_s L^s \tilde g_T \right) ||_\cH \leq  
c_{a,s, \alpha, \kappa_s }\log^2(8/\delta)(\gamma T)^{-\alpha}  \left( d_{s}(R)   + R(\gamma T)^{-\beta}\right)\;, \]
for some  $c_{a,s, \alpha, \kappa_s  }<\infty$.
\end{enumerate}
\end{proposition}

\begin{proof}[Proof of Proposition \ref{prop:IIII}]
\begin{enumerate}
\item
We write 
\begin{align*}
& ||\cT_s^\alpha \bar G_T(\cT_{\bx , s})\left(  \cB_s^*\cS f_\cH - \cT_s L^s \tilde g_T \right) ||_\cH =  
||\cT_s^\alpha \bar G_T(\cT_{\bx , s})\cT_s L^s \left(   f_\cH -  \tilde g_T \right) ||_\cH \no \\
&\leq ||(\cT_s+ \lambda)^\alpha (\cT_{\bx , s} + \lambda )^{-\alpha} || \; 
  ||(\cT_{\bx , s} + \lambda )^{\alpha} \bar G_T(\cT_{\bx , s})(\cT_{\bx , s} + \lambda )^{1-\alpha}|| \no \\ 
& \quad ||(\cT_{\bx , s} + \lambda )^{-(1-\alpha)}(\cT_s+ \lambda)^{1-\alpha}||\; ||(\cT_s+ \lambda)^{-(1-\alpha)} \cT_s L^s \left(   f_\cH -  \tilde g_T \right)||_\cH \no \\
 &\leq ||(\cT_s+ \lambda)^\alpha (\cT_{\bx , s} + \lambda )^{-\alpha} || \; 
  ||(\cT_{\bx , s} + \lambda )^{\alpha} \bar G_T(\cT_{\bx , s})(\cT_{\bx , s} + \lambda )^{1-\alpha}|| \no \\ 
& \quad ||(\cT_{\bx , s} + \lambda )^{-(1-\alpha)}(\cT_s+ \lambda)^{1-\alpha}||\; ||(\cT_s+ \lambda)^\alpha L^s \left(   f_\cH -  \tilde g_T \right)||_\cH \;. 
\end{align*}
Recall that by the Cordes Inequality, since $\alpha \in (0, \frac{1}{2}]$, 
\[   ||(\cT_s+ \lambda)^\alpha (\cT_{\bx , s} + \lambda )^{-\alpha} || \leq  ||(\cT_s+ \lambda)(\cT_{\bx , s} + \lambda )^{-1} ||^\alpha \]
and thus 
\[  ||(\cT_s+ \lambda)^\alpha (\cT_{\bx , s} + \lambda )^{-\alpha} || \; ||(\cT_{\bx , s} + \lambda )^{-(1-\alpha)}(\cT_s+ \lambda)^{1-\alpha}|| \leq 
||(\cT_s+ \lambda)(\cT_{\bx , s} + \lambda )^{-1} || \;.\]
Moreover, \cite[Lemma 2]{mucke2019beating} with $K=3$  gives $\rho_X^{\otimes n}$-almost surely 
\[  ||(\cT_{\bx , s} + \lambda )^{\alpha} \bar G_T(\cT_{\bx , s})(\cT_{\bx , s} + \lambda )^{1-\alpha}|| = 
||(\cT_{\bx , s} + \lambda ) \bar G_T(\cT_{\bx , s})|| \leq 3 \;.\]
Applying now Proposition \ref{prop:II} leads us to 
\begin{align*}
 ||\cT_s^\alpha \bar  G_T(\cT_{\bx , s})\left(  \cB_s^*\cS f_\cH - \cT_s L^s \tilde g_T \right) ||_\cH &\leq 
c_{a,s, \alpha }\Xi_{\bx,s}(\lambda )(\gamma T)^{-\alpha}  \left( d_{s}(R)   + R(\gamma T)^{-\beta}\right)\;,
\end{align*}
for some $c_{a,s, \alpha }<\infty$ and with  $\beta=\frac{a-s}{2(a+s)}$. 
\item 
The second assertion follows immediately from Proposition \ref{prop:sub-proba}. 
\end{enumerate}
\end{proof}

\[\]

\begin{assumption}[Moment Assumption]
\label{ass:moment}
There exist $M \in \mbr_+$, $\sigma \in \mbr_+$ such that for any integer $l\geq 2$ 
\[ \int_\cY |y|^l\; d_\rho(y|x) \leq \frac{1}{2} l! M^{l-2}\sigma^2 \;, \]
$\rho_X$-almost surely.
\end{assumption} 

\[\]

\begin{proposition}
\label{lem:proba-error12}
Suppose Assumptions \ref{ass:bounded}, \ref{ass:moment} are satisfied. Let $\delta \in (0,1]$ and set 
\[  \hat h_T := (\cB^*_{\bx , s} \bar y -\cB_s^*\cS f_\cH )-   (\cT_{\bx , s} - \cT_s )L^s \tilde g_T   \;. \]
Then, with probability $\rho^{\otimes n}$ at least $1-\delta$ one has 
\[ || (\cT_s + \lam )^{-1/2}\hat h_T||_\cH \leq  4\log(2/\delta ) \left( \frac{\sqrt{\gamma T}\tilde M_T}{n} + 
\tilde \sigma  \sqrt{\frac{\cN_{\cT_s}(1/(\gamma T)) }{n}} + \tilde \tau_T \sqrt{\frac{\gamma T}{n}}    \right)\;, \]
where 
\begin{align}
\label{eq:tildeM}
 \tilde M_T &:= 2\kappa_s (M + \kappa_s  ||L^s\tilde g_T||_\cH ) \;,
\end{align} 
and 
\begin{equation}
\label{eq:tildesigma}
   \tilde \sigma := \sigma + 2||\cS f_\cH||_\infty \;, \quad \tilde \tau_T :=  2\kappa_s ||S(\tilde g_T - f_{\cH}) ||_{L^2} \;. 
\end{equation}   
\end{proposition}

\begin{proof}[Proof of Proposition \ref{lem:proba-error12}]
The proof  follows by repeating the arguments of \cite[Lemma 5.6]{Lin18} with a slight adaption to our setting. 
\end{proof}

\[\]

\begin{proposition}\label{prop:V}
Suppose Assumptions \ref{ass:bounded}, \ref{ass:link-condition3}, and  \ref{ass:moment} are satisfied. Assume further that $\gamma$ and $ T$ satisfy the 
condition
\[  n \geq \gamma T \max\{1, \cN_{\cT_s}((\gamma T)^{-1})\}    \;.   \]  
Let $\delta \in (0,1]$ and set   
\[ \hat   h_T :=   (\cB^*_{\bx , s} \bar y -\cB_s^*\cS f_\cH )-   (\cT_{\bx , s} - \cT_s )L^s \tilde g_T \;. \]
Then, with probability $\rho^{\otimes n}$ at least $1-\delta$ one has 
\[ ||\cT_s^\alpha \bar G_T(\cT_{\bx , s}) \hat  h_T ||_\cH \leq c_{\kappa_s} \log^{2\alpha +1}(8/\delta )(\gamma T)^{ - \alpha}  \left( \tilde M_T \frac{\gamma T}{n} + 
\tilde \sigma  \sqrt{\frac{\gamma T \cN_{\cT_s}(1/(\gamma T)) }{n}} + \tilde \tau_T \frac{\gamma T}{\sqrt n}    \right) \;, \]
for some $  c_{\kappa_s}< \infty$ and where $\tilde M_T $, $\tilde \sigma$, $ \tilde \tau_T$ are defined in Proposition \ref{lem:proba-error12}.
\end{proposition}

\begin{proof}[Proof of Proposition \ref{prop:V}]
We decompose with $\lam = (\gamma T)^{-1}$ as 
\begin{align*}
& ||\cT_s^\alpha \bar G_T(\cT_{\bx , s})\hat v_T ||_\cH \no \\ 
&\leq ||(\cT_s + \lam)^\alpha (\cT_{\bx , s} + \lam)^{-\alpha }|| \; ||(\cT_{\bx , s} + \lam)^{\alpha } \bar G_T(\cT_{\bx , s})(\cT_{\bx , s} + \lam)^{1/2}|| \no \\
   &\quad ||(\cT_{\bx , s} + \lam)^{-1/2}(\cT_s + \lam)^{1/2} || \;||(\cT_s + \lam)^{-1/2}v_T ||_\cH \;.
\end{align*}
Observe that \cite[Lemma 2]{mucke2019beating} with $K=3$ gives 
\[ ||(\cT_{\bx , s} + \lam)^{\alpha + 1/2} \bar G_T(\cT_{\bx , s})|| \leq 3\cdot 2^{\alpha-1/2} (\gamma T)^{1/2 - \alpha} \;.\]
Thus, since $\alpha \leq 1/2$ by applying Proposition \ref{lem:proba-error12} and Proposition \ref{prop:sub-proba}, we obtain 
with probability $\rho^{\otimes n}$ at least $1-\delta$  
\begin{align*}
 ||\cT_s^\alpha \bar G_T(\cT_{\bx , s}) \hat   h_T  ||_\cH 
&\leq 3 \Xi^\alpha_{\bx,s}(\lambda )\;\Xi^{1/2}_{\bx,s}(\lambda )(\gamma T)^{1/2 - \alpha} \;||(\cT_s + \lam)^{-1/2} \hat   h_T  ||_\cH \no \\
&\leq c_{\kappa_s} \log^{2\alpha +1}(8/\delta )(\gamma T)^{ - \alpha}  \left( \tilde M_T \frac{\gamma T}{n} + 
\tilde \sigma  \sqrt{\frac{\gamma T \cN_{\cT_s}(1/(\gamma T)) }{n}} + \tilde \tau_T \frac{\gamma T}{\sqrt n}    \right)\;,
\end{align*}
for some $  c_{\kappa_s}< \infty$.
\end{proof}


\subsubsection{Learning Rates for smoothness promoting tail-averaged GD}

We now combine the results from  the previous subsections to derive the learning rates for smoothness promoting tail-averaged gradient descent.

\begin{theorem}[Excess Risk tail-averaged GD]
\label{theo:GDI}
Suppose Assumptions \ref{ass:bounded}, \ref{ass:link-condition3} and  \ref{ass:moment} are satisfied. 
Let $R>0$ and $\beta=\frac{a-s}{2(a+s)}$. Assume further that $\gamma$ and $ T$ satisfy the 
condition
\[  n \geq \gamma T \max\{1, \cN_{\cT_s}((\gamma T)^{-1})\}   \;.\]
If $f_\cH \in \cD(L^s)$,  the excess risk satisfies for any $\alpha \in [0,\frac{1}{2}]$
\begin{align}
\label{eq:final-GD}
 \mbe\left[ || \cT_s^\alpha L^s (\bar g_T  - f_\cH) ||_{L^2} \right ] &\leq 
c_{a,s,\alpha, \kappa_s} (\gamma T)^{-\alpha} \left(\; d_{s}(R)   + R(\gamma T)^{-\beta} + \frac{|| L^s\tilde g_T ||_\cH }{\gamma T}  +  \right.  \no\\
&  \left. +  \tilde M_T \frac{\gamma T}{n} + \tilde \sigma  \sqrt{\frac{\gamma T \cN_{\cT_s}(1/(\gamma T)) }{n}} + \tilde \tau_T \frac{\gamma T}{\sqrt n}  \;  \right) \;, 
\end{align} 
for some $c_{a,s,\alpha, \kappa_s}<\infty$ and where $\tilde M_T $, $\tilde \sigma$, $ \tilde \tau_T$ are defined in Proposition \ref{lem:proba-error12}. 
In particular, 
\begin{align*}
 \mbe\left[ || \bar g_T - f_\cH ||_\cH \right ] &\leq 
 c_{a,s,\kappa_s} (\gamma T)^{-\frac{s}{2(a+s)}} \left(\; d_{s}(R)   + R(\gamma T)^{-\beta} + \frac{|| L^s\tilde g_T ||_\cH }{\gamma T}  +  \right.  \no\\
&  \left. +  \tilde M_T \frac{\gamma T}{n} + \tilde \sigma  \sqrt{\frac{\gamma T \cN_{\cT_s}(1/(\gamma T)) }{n}} + \tilde \tau_T \frac{\gamma T}{\sqrt n}  \;  \right) \;, 
\end{align*}
\begin{align*}
 \mbe\left[ || \cS (\bar g_T - f_\cH ) ||_{L^2} \right ] &\leq c'_{a,s,\kappa_s} (\gamma T)^{-\frac{1}{2}} \left(\; d_{s}(R)   + R(\gamma T)^{-\beta} + 
\frac{ || L^s\tilde g_T ||_\cH }{\gamma T}  +  \right.  \no\\
&  \left. +  \tilde M_T \frac{\gamma T}{n} + \tilde \sigma  \sqrt{\frac{\gamma T \cN_{\cT_s}(1/(\gamma T)) }{n}} + \tilde \tau_T \frac{\gamma T}{\sqrt n}  \;  \right) \;,
\end{align*}
and 
\begin{align*}
 \mbe\left[ || \bar g_T - f_\cH ||_{s} \right ] &\leq \tilde c_{a,s,\kappa_s}  \left(\; d_{s}(R)   + R(\gamma T)^{-\beta} + 
\frac{ || L^s\tilde g_T ||_\cH }{\gamma T}  +  \right.  \no\\
&  \left. +  \tilde M_T \frac{\gamma T}{n} + \tilde \sigma  \sqrt{\frac{\gamma T \cN_{\cT_s}(1/(\gamma T)) }{n}} + \tilde \tau_T \frac{\gamma T}{\sqrt n}  \;  \right) \;,
\end{align*}
\end{theorem}

\begin{proof}[Proof of Theorem \ref{theo:GDI}]
The proof follows immediately from the decomposition in \eqref{eq:I} and \eqref{eq:II}, combined with 
Propositions \ref{prop:II},  \ref{prop:III},  \ref{prop:IIII} and \ref{prop:V}. The bound in expectation is derived by integrating the high probability bound using e.g. 
\cite[Lemma C.1]{blanchard2018optimal}.
\end{proof}

\[\]

\begin{lemma}
\label{lem:approxbounds}
Suppose Assumptions \ref{ass:bounded} and \ref{ass:link-condition3}  are satisfied. Recall the definition of the distance function $d_s$ in 
Definition \ref{def:distance1}.  Assume further that $f_\cH \in \cH_s$ and let $R>0$. 
\begin{enumerate}
\item We have  
\[ ||S(\tilde g_T - f_\cH) ||_{L^2} \leq 2 (\gamma T)^{-1/2}(d_s(R) + R(\gamma T)^{-\frac{a-s}{2(a+s)}})  \;.\]
\item Let $f_\cH \in \cD(L^s)$ and $\beta = \frac{a-s}{2(a+s)}$. Then 
\[  ||L^s\tilde g_T||_\cH \leq  c_{a,s}(d_s(R) + R(\gamma T)^{-\beta}) +||L^s f_\cH||_\cH  \;,\]
for some $c_{a,s} < \infty$.
\end{enumerate}
\end{lemma}

\begin{proof}[Proof of Lemma \ref{lem:approxbounds}]
\begin{enumerate}
\item
Recall that according to Lemma \ref{lem:prelim3} we have $L^sf_{R} = \cT_s^{\beta} h$ with $\beta=\frac{a-s}{2(a+s)}$ and 
for some $ h\in \cH$, satisfying $|| h|| \leq R$. 
By \cite[Lemma 3]{mucke2019beating}  we thus obtain 
\begin{align*}
 || S(\tilde g_T - f_\cH ) ||_{L^2} &=  ||\sqrt{\cT_s} \bar R_T(\cT_s)L^s f_\cH||_{L^2} \no \\ 
 &\leq ||\sqrt{\cT_s} \bar R_T(\cT_s)L^s (f_\cH - w_{R, \cH})||_{L^2}+ ||\sqrt{\cT_s} \bar R_T(\cT_s)L^s  f_{R}||_{L^2} \no \\ 
 &\leq \sqrt 2 (\gamma T)^{-1/2} d_s(R) + R|| \bar R_T(\cT_s)\cT_s^{1/2 + \beta}||_{L^2}\no \\ 
 &\leq  2 (\gamma T)^{-1/2}(d_s(R) + R(\gamma T)^{-\beta}) \;.
\end{align*}
\item For any $f_\cH \in \cD(L^s)$ we have by applying Proposition \ref{prop:II} with $\alpha =0 $
\begin{align*}
 ||L^s\tilde g_T||_\cH &\leq  ||L^s(\tilde g_T - f_\cH)||_\cH  + ||L^s f_\cH||_\cH \\
&\leq   c_{a,s}(d_s(R) + R(\gamma T)^{-\beta}) +||L^s f_\cH||_\cH  \;,
\end{align*}
for some $c_{a,s} < \infty$.
\end{enumerate}
\end{proof}

\begin{corollary}[Learning Rates GD I]
\label{cor:rates-promo-GD-I}
Suppose all assumptions of Theorem \ref{theo:GDI} are satisfied with $f_\cH \in \cH_s$ for some $0 \leq s < a$. In addition, suppose 
Assumption \ref{ass:effective-dim} is satisfied. Then, for any $n$ sufficiently large, the excess risk satisfies for any $\alpha \in [0, \frac{1}{2}]$ 
the bound 
\[ \mbe\left[ ||\cT_s^\alpha L^s ( \bar g_T - f_\cH )||^2_{L^2} \right]   \lesssim R_*^2 \left( \frac{\tilde \sigma^2}{R_*^2  n}  \right)^{\frac{2\alpha}{1+\nu } }\;, \] 
for the following choices for $\gamma$ and $T$:
\begin{enumerate}
\item[(a)] small stepsize:  $T_n \simeq n$, $\gamma_n \simeq (R_*^2/\tilde \sigma^2) \left( \frac{R_*^2}{\tilde \sigma^2}n\right)^{\frac{\nu}{1+\nu}}$.
\item[(b)] early stopping: $\gamma \simeq const.$ and $T_n \simeq \left( \frac{R_*^2}{\tilde \sigma^2}n\right)^{\frac{1}{1+\nu}}$.
\end{enumerate}
In particular, with $\beta=\frac{a-s}{2(a+s)}$, we obtain 
\[ \mbe\left[ ||\cS( \bar g_T - f_\cH )||^2_{L^2} \right]   \lesssim R_*^2 \left( \frac{\tilde \sigma^2}{R_*^2  n}  \right)^{\frac{1}{1+\nu } }\;, \]
\[ \mbe\left[ || \bar g_T - f_\cH ||^2_{\cH} \right]   \lesssim R_*^2 \left( \frac{\tilde \sigma^2}{R_*^2  n}  \right)^{\frac{2\beta}{1+\nu } }\;, \]
for each of the choices given in $(a)$ and $(b)$.
\end{corollary}

\begin{proof}[Proof of Corollary \ref{cor:rates-promo-GD-I}]
According to \eqref{eq:bound-distance} (see also the remark below), the distance function is bounded as 
\[ d_s(R) \leq R_*\;, \quad R_{\gamma T} \leq R_*(\gamma T)^\beta \;,\]
for some $R_* > 0$. Since 
\[  \cN_{\cT_s}((\gamma T)^{-1} \lesssim (\gamma T)^\nu \;,  \]
deriving the learning rate amounts to balance the terms 
\[  R_*   + R_*(\gamma T)^\beta (\gamma T)^{-\beta}  + \tilde \sigma  \sqrt{\frac{(\gamma T)^{1+\nu} }{n}} \;,    \]
in \eqref{eq:final-GD}. This leads to the choice 
\[ \gamma T \asymp \left(\frac{R^2_*}{\tilde \sigma^2} n \right)^{\frac{1}{1+\nu}} \]
and immediately to  $(a)$ and $(b)$. 
From Lemma \ref{lem:approxbounds} we see that the other terms in \eqref{eq:final-GD} are of lower order.
\end{proof}

\begin{corollary}[Learning Rates GD II]
\label{cor:rates-promo-GD-II}
Suppose all assumptions of Theorem \ref{theo:GDI} are satisfied with $f_\cH \in \cH_a$. In addition, suppose 
Assumption \ref{ass:effective-dim} is satisfied. Then, for any $n$ sufficiently large, the excess risk satisfies 
\[ \mbe\left[ ||\cT_s^\alpha L^s ( \bar g_T - f_\cH )||^2_{L^2} \right]   \lesssim R_*^2 \left( \frac{\tilde \sigma^2}{R_*^2  n}  \right)^{\frac{2(\alpha+\beta)}{1+2\beta + \nu }}\;, \] 
for the following choices for $\gamma$ and $T$:
\begin{enumerate}
\item[(a)] small stepsize:  $T_n \simeq n$, $\gamma_n\simeq  \frac{R_*^2}{\tilde \sigma^2} \left(\frac{\tilde \sigma^2}{R_*^2 n} \right)^{\frac{\nu +2\beta}{1+2\beta + \nu}}$.
\item[(b)] early stopping: $\gamma \simeq const.$ and $T_n \simeq \left( \frac{R_*^2}{\tilde \sigma^2}n\right)^{\frac{1}{1+2\beta + \nu}}$.
\end{enumerate}
In particular, we obtain 
\[ \mbe\left[ ||\cS( \bar g_T - f_\cH )||^2_{L^2} \right]   \lesssim  R_*^2 \left( \frac{\tilde \sigma^2}{R_*^2  n}  \right)^{\frac{1+2\beta }{1+2\beta + \nu }}\;, \]
\[ \mbe\left[ || \bar g_T - f_\cH ||^2_{\cH} \right]   \lesssim  R_*^2 \left( \frac{\tilde \sigma^2}{R_*^2  n}  \right)^{\frac{4\beta}{1+2\beta + \nu }}\;, \]
for each of the choices given in $(a)$ and $(b)$.
\end{corollary}

\begin{proof}[Proof of Corollary \ref{cor:rates-promo-GD-II}]
If $f_\cH \in \cH_a$ then $d_s(R_*)=0$ for some $R_*>0$. The rest follows as in the proof of Corollary \ref{cor:rates-promo-GD-I}.
\end{proof}


\subsection{Error Bounds for Preconditioning Gradient Descent}

We now derive our error bounds for the case where  $f_\cH  \in \cH_s$, for some $-a/2 \leq s \leq 0$.  The result again relies on the classical 
decomposition \eqref{eq:classical}, with a slight adaption of the approximation error bound.

\subsubsection{Bounding the Approximation Error}

\begin{proposition}\label{prop:preII}
Suppose Assumptions \ref{ass:bounded}, \ref{ass:link-condition2} and \ref{ass:SC-pre} are satisfied. Let $R>0$. For 
$\beta= \frac{r-s}{2(a+s)}$ and for any $\alpha \in [0, \frac{1}{2}]$ one has 
\begin{enumerate}
\item
\[  ||\cT_s^\alpha L^s(\tilde g_T -   f_\cH )||_\cH \leq c_{\alpha, \beta} R (\gamma T)^{-(\alpha + \beta)}  \;, \]
for some $c_{\alpha, \beta} < \infty$. 
\item
Additionally, with $\lam = (\gamma T)^{-1}$, one has 
\[  ||(\cT_s + \lam )^\alpha L^s(\tilde g_T -   f_\cH )||_\cH \leq   c'_{\alpha, \beta} R(\gamma T)^{-(\alpha + \beta)}  \;, \]
for some $c'_{\alpha, \beta} < \infty$. 
\item Moreover, we have 
\[ ||L^s\tilde g_T||_\cH \leq R (\gamma T)^{-\beta} + ||L^s f_\cH ||_\cH\;. \]
\end{enumerate}
\end{proposition}

\begin{proof}[Proof of Proposition \ref{prop:preII}]
\begin{enumerate}
\item Applying \cite[Lemma 3]{mucke2019beating} with $K=3$  we get  
\begin{align*}
 ||\cT_s^\alpha L^s(\tilde g_T -   f_\cH )||_\cH &\leq   ||\cT_s^\alpha \bar R_T(\cT_s)  L^s f_\cH ||_\cH  \no \\
  &=  ||\cT_s^{\alpha+\beta} \bar R_T(\cT_s)   ||_\cH  \no \\  
&\leq c_{\alpha, \beta} R (\gamma T)^{-(\alpha + \beta)} \;,                             
\end{align*}
for some $c_{\alpha, \beta} < \infty$. 
\item The second assertion is a standard calculation, using $(i)$ and \cite[Lemma 3]{mucke2019beating}. 
\item The proof is along the lines of the proof of Lemma \ref{lem:approxbounds}. 
\end{enumerate}
\end{proof}

\subsubsection{Bounding the Estimation Error}

\begin{proposition}\label{prop:IIIIPre}
Suppose Assumptions \ref{ass:bounded}, \ref{ass:link-condition2}, and \ref{ass:SC-pre} are satisfied. Let $\beta = \frac{r-s}{2(a+s)}$. 
Assume $\gamma$ and $ T$ satisfy the 
condition
\[   n \geq \gamma T \max\{1, \cN_{\cT_s}((\gamma T)^{-1}) \} \;.  \]
Let $\alpha \in [0, \frac{1}{2}]$.  For any $\delta \in (0,1]$ with  probability $\rho^{\otimes n}$ at least $1-\delta$ 
\[ ||\cT_s^\alpha G_T(\cT_{\bx , s})\left(  \cB_s^*\cS f_\cH - \cT_s L^s \tilde g_T \right) ||_\cH \leq 
c_{a,s, \alpha, \kappa_s  }  R \; \log^2(8/\delta) (\gamma T)^{-(\alpha + \beta)} \;, \]
for some  $c_{a,s, \alpha, \kappa_s  }<\infty$.
\end{proposition}

\begin{proof}[Proof of Proposition \ref{prop:IIIIPre}]
The proof follows the lines of Proposition \ref{prop:IIII} by applying Proposition \ref{prop:preII}. 
\end{proof}

Since Proposition \ref{prop:V} is still valid we can give now the error bound also in the preconditioning case.

\subsubsection{Learning Rates  Preconditioning GD}

\begin{theorem}
\label{theo:GDIPre}
Suppose Assumptions \ref{ass:bounded}, \ref{ass:link-condition2}, \ref{ass:SC-pre} and  \ref{ass:moment} are satisfied. 
Assume further that $\gamma$ and $ T$ satisfy the 
condition
\[    n \geq \gamma T \max\{1, \cN_{\cT_s}((\gamma T)^{-1}) \} \;.   \]  
The excess risk satisfies for any $\alpha \in [0,\frac{1}{2}]$
\begin{align*}
 \mbe\left[ || \cT_s^\alpha L^s (\bar g_T  - f_\cH) ||_{L^2} \right ] &\leq 
c_{a,s,\alpha, \kappa_s} (\gamma T)^{-\alpha} \left(\; R (\gamma T)^{-\frac{r-s}{2(a+s)}} + \frac{||L^s \tilde g_T||_\cH}{\gamma T}  \right. \\
 & +  \left.      \tilde M_T \frac{\gamma T}{n} + \tilde \sigma  \sqrt{\frac{\gamma T \cN_{\cT_s}(1/(\gamma T)) }{n}} + \tilde \tau_T \frac{\gamma T}{\sqrt n}   \;  \right) \;, 
\end{align*} 
for some $c_{a,s,\alpha, \kappa_s}<\infty$ and where $\tilde M_T $, $\tilde \sigma$, $ \tilde \tau_T$ are defined in Proposition \ref{lem:proba-error12}. 
\end{theorem}


\begin{corollary}[Learning rates]
Let all assumptions of Theorem \ref{theo:GDIPre} be satisfied as well as Assumption \ref{ass:effective-dim}. Then the excess risk satisfies the same bounds as in 
Corollary \ref{cor:rates-promo-GD-II} with $\beta = \frac{r-s}{2(a+s)}$. 
\end{corollary}





\section{Error Bounds for SGD in Hilbert Scales}
\label{app-error-SGD}
In this section we derive the learning rates for the tail-averaged SGD recursion in Hilbert scales. As usual, the error bounds rely on 
a suitable error decomposition. Here, we follow the approach in \cite{pillaud2018statistical, mucke2019beating} and consider 
\begin{equation}
\label{eq:decompfinalSGD}
\bar f_T - f_\cH = \underbrace{(\bar f_T - \bar g_T)}_{SGD \;\; variance} + \underbrace{(\bar g_T - f_\cH)}_{Section \;\; \ref{app-error-analysis-smoothness-promoting}} \;,
\end{equation}
where $\bar g_T$ is the tail-averaged GD recursion, defined in \eqref{eq:barGD}. 
Recall that the GD recursion defined in \eqref{eq:barGD0} can be rewritten as 
\begin{align*}
   g_{t} &= L^{-s} h^{GD}_{t}\,, \quad h^{GD}_{t+1} = h^{GD}_{t} - \gamma(\cT_{\bx, s}h^{GD}_{t}  - \cB^*_{\bx, s}\by ) \;.
\end{align*}
Denoting $(\bx_t, \by_t):=(x_{j_i}, y_{j_i})_{i=b(t-1)+1, ..., bt} \in (\cX \times \cY)^b$ the sample drawn at iteration $t$ , $t=1,...,T$ and 
\[ \cB_{t,s} :=\cS_{\bx_t, s}L^{-s}: \cH \to \mbr^b  \;, \;\;\; \cT_{t,s}:= \cB^*_{t,s}\cB_{t,s} : \cH \to \cH \;,\]
a short calculation shows that the SGD recursion in \eqref{eq:SGD-recursion} can similarly be rewritten as 
\[ f_{t+1} = L^{-s}h^{SGD}_{t} \,, \quad h^{SGD}_{t+1} = h^{SGD}_{t} - \gamma(\cT_{t+1, s}h^{SGD}_{t}  - \cB^*_{t+1, s}\by_{t+1} ) \;. \]
Thus, the difference 
$$\mu_t:=f_t - g_t=L^{-s}\tilde \mu_t$$  
with $\tilde \mu_t = h^{SGD}_{t}-h^{GD}_{t}$ follows the recursion 
\begin{align*}
\tilde \mu_{t+1} &= (1-\gamma \cT_{t+1,s})\tilde \mu_t + \gamma   \xi_{t+1, s}   \;,
\end{align*}
where $ \xi_{t+1, s}=\xi^{(1)}_{t+1, s} + \xi^{(2)}_{t+1, s} $ and with 
\begin{align*}
\xi^{(1)}_{t+1, s} &=(\cT_{\bx,s} - \cT_{t+1,s} ) h^{GD}_{t} \; \\
\xi^{(2)}_{t+1, s} &=  \cS^*_{t+1,s} \by_{t+1} -  \cS^*_{\bx ,s} \by \;, 
\end{align*}
see \cite[Appendix E]{mucke2019beating}. We finally set $\bar \mu_T = L^{-s} \overline{\tilde \mu}_T$, with 
$\overline{\tilde \mu}_T = \frac{2}{T}\sum_{t=\floor{T/2}+1}^T \tilde \mu_t$.

\subsection{Bounding SGD Variance}

\begin{proposition}
\label{prop.new}
Suppose Assumptions \ref{ass:bounded} and \ref{ass:link-condition3} are satisfied. 
Let $s \geq 0$, $\nu \in [0, 1]$, $\gamma \kappa_s^2 < \frac{1}{4}$, $\beta=\frac{a-s}{2(a+s)}$ and $f_\cH \in \cD(L^s)$. 
Define
\begin{equation}
\label{eq:def-SigmaT}
  \Sigma_T^2 (R)= \tilde M^2+||L^s f_\cH ||^2_\cH + \sigma^2_T (R) 
\end{equation}  
and 
\[  \sigma_T (R) = d_s(R) + R(\gamma T)^{-\beta} + \tilde \sigma + M + \frac{||L^sf_\cH ||_\cH}{\gamma T} + (\gamma T)^{-1} (d_s(R)+ R(\gamma T)^{-\beta} )  \;.\]
Assume further that 
\[ n \geq \gamma T \max\{ 1, \cN_{\cT_s}(1/\gamma T) \} \;.\]
The SGD variance satisfies 
\[ \mbe\left[ || \sqrt{\cT_s }\overline{\tilde \mu}_T ||^2_\cH \right] \leq  c_s Tr[\cT_s^{\nu}]  \; \frac{\Sigma^2_T (R)}{b} 
            \gamma^{\nu}T^{\nu - 1}  
     +   c_s' \gamma^2 T^5 \delta_n    \;, \]
for some $c_s < \infty$, $c'_s < \infty$ and with 
\begin{equation}
\label{eq:delta}
  \delta_n = 2\exp\left( -a \sqrt{\frac{n}{\gamma T \cN_{\cT_s}(1/\gamma T)}} \right) \;,
\end{equation}  
for some $a \in \mbr_+$. 

The same bound holds if $-a/2 \leq s \leq 0$ under Assumption \ref{ass:link-condition2} instead of Assumption \ref{ass:link-condition3}  and with $\beta=\frac{r-s}{2(a+s)}$,  
provided Assumption \ref{ass:SC-pre} is satisfied. 
\end{proposition}

\begin{proof}[Proof of Proposition \ref{prop.new}]
The proof is along the lines of the proof of \cite[Proposition 6, Section D.1]{mucke2019beating} with slight modifications adapted to our setting. 
Note that Assumption \ref{ass:moment} is satisfied if $\cY \subseteq [-M,M]$, for some $M>0$, with $\sigma = M$. The main new ingredient in the proof is a new 
uniform bound  for the GD updates $||h^{GD}_{t}||_\cH = ||L^s g_t||_\cH = ||g_t||_s$  based on our Theorem \ref{theo:GDI}. This is straightforward by 
following the arguments as in \cite[Lemma 8, Section D.1]{mucke2019beating}.
\end{proof}


\subsection{Final Bounds smoothness promoting SGD} 
\label{app-proofs-SGD}

\begin{theorem}[Excess Risk]
Suppose Assumptions \ref{ass:bounded} and \ref{ass:link-condition3} are satisfied. 
Let $\nu \in (0,1]$, $\beta=\frac{a-s}{2(a+s)}$ and $\gamma \kappa_s^2 < \frac{1}{4}$. Assume further that $ Tr[\cT_s^{ \nu}] <\infty$ and that 
\[   n \geq \gamma T \max\{ 1, \cN_{\cT_s}(1/\gamma T) \}  \;.   \]  
If $f_\cH \in \cD(L^s)$,  the excess risk satisfies 
\begin{align*}
 \mbe\left[ ||\cS( \bar f_T - f_\cH )||^2_{L^2} \right] &\leq c_{a,s,\kappa_s, \nu} (\gamma T)^{-1 } \left(\; d^2_{s}(R)   + R^2(\gamma T)^{-2\beta}   
      + \tilde \sigma^2  \frac{\gamma T \cN_{\cT_s}(1/(\gamma T)) }{n}     \;\right)  \no \\
   & \quad   +  \frac{\Sigma^2_T (R)}{bT}  (\gamma T)^{\nu}    + \Omega^2_n(\gamma T)  \;, 
\end{align*}
where the remainder $\Omega_n(\cdot)$ is of lower order and given by 
\begin{align*}
 \Omega_n(\gamma T) 
 &= \frac{ || L^s\tilde g_T ||_\cH }{\gamma T} +   \sqrt{\frac{\gamma T}{n}}(d_s(R) + R^2(\gamma T)^{-\beta}) +  \gamma^2 T^5 \delta_n
\end{align*}
and where $\Sigma^2_T(R)$ is defined in \eqref{eq:def-SigmaT} and $\delta_n$ is given by \eqref{eq:delta}.
\end{theorem}

\[\]

\begin{proof}[Proof of Theorem \ref{prop:newSGD}]
The proof follows by combining \eqref{eq:decompfinalSGD} with Proposition \ref{prop.new} and Theorem \ref{theo:GDI} with $\alpha = 1/2$. 
Moreover, 
Lemma \ref{lem:approxbounds} gives  
\begin{equation}
\label{eq:SGD1}
  ||L^s\tilde g_T||_\cH \leq  c_{a,s}(d_s(R) + R(\gamma T)^{-\beta}) +||L^s f_\cH||_\cH  \;,
\end{equation}  
for some $c_{a,s} < \infty$. Thus, 
\begin{align}
\label{eq:SGD2}
 \tilde M_T &= 2\kappa_s (M + \kappa_s ||L^s f_\cH||_\cH ) \no \\
  &\leq c_{a,\kappa_s} (M+  d_s(R) + R(\gamma T)^{-\beta} + ||L^s f_\cH||_\cH ) \;,
\end{align}
for some $c_{a,\kappa_s} < \infty$. 
In addition, applying Lemma \ref{lem:approxbounds} once more leads to 
\begin{align}
\label{eq:SGD3}
\tilde \tau_T  &= 2\kappa_s ||\cS \tilde g_T - f_\cH||_{L^2} \no \\
&\leq 4 \kappa_s (\gamma T)^{-1/2}(d_s(R) + R(\gamma T)^{-\beta}) \;.
\end{align}
Combining  Theorem \ref{theo:GDI} and Proposition \ref{prop.new}  with \eqref{eq:SGD1}, \eqref{eq:SGD2}, and \eqref{eq:SGD3} leads to 
\begin{align*}
 & \mbe\left[ ||\cS \bar f_T - f_\cH ||^2_{L^2} \right] \leq c_{a,s,\kappa_s} (\gamma T)^{-\frac{1}{2}} \left(\; d_{s}(R)   + R(\gamma T)^{-\beta} + 
\frac{ || L^s\tilde g_T ||_\cH }{\gamma T}  +  \right.  \no\\
&  \left. +  \tilde M_T \frac{\gamma T}{n} + \tilde \sigma  \sqrt{\frac{\gamma T \cN_{\cT_s}(1/(\gamma T)) }{n}} + \tilde \tau_T \frac{\gamma T}{\sqrt n}  \;  \right) \\
  & + c'_{\kappa_s} \left(   Tr[\cT_s^{ \nu}]  \; \frac{\Sigma^2_T (R)}{b} 
            \gamma^{\nu}T^{\nu - 1}       +   c_s' \gamma^2 T^5 \delta_n  \right)   \\
 &\leq c_{a,s,\kappa_s, \nu} (\gamma T)^{-\frac{1}{2}} \left(\; d_{s}(R)   + R(\gamma T)^{-\beta}   
      + \tilde \sigma  \sqrt{\frac{\gamma T \cN_{\cT_s}(1/(\gamma T)) }{n}}       +  \right.  \no \\
   & \left. +  \frac{\Sigma^2_T (R)}{bT}  (\gamma T)^{\nu}    + \Omega_n(\gamma T)  \;  \right)\;.  
\end{align*}
\end{proof}

\begin{proof}[Proof of Corollary \ref{cor:ratesSGD}]
The proof of the  learning rate is now a standard calculation. A straightforward calculation shows that all choices of $\gamma, T, b$ balance the 
leading order terms in Theorem \ref{prop:newSGD} and that $\Omega_n(\cdot)$ is of lower order. 
\end{proof}

\begin{proof}[Proof of Corollary \ref{cor:ratesSGD2}]
The proof of the  learning rate is now a standard calculation. A straightforward calculation shows that all choices of $\gamma, T, b$ balance the 
leading order terms in Theorem \ref{prop:newSGD} and that $\Omega_n(\cdot)$ is of lower order. Moreover, $d_s(R)\leq R_*$ and $R(\gamma T)\asymp R_*(\gamma T)^\beta $, 
for some $R_*>0$. 
\end{proof}

\subsection{Final Bounds Preconditioning  SGD }

\begin{theorem}[Excess Risk]
\label{prop:newSGDPre}
Suppose Assumptions \ref{ass:bounded}, \ref{ass:link-condition2}, and  \ref{ass:SC-pre} are satisfied. 
Let $\nu \in (0,1]$, $\beta=\frac{r-s}{2(a+s)}$ and $\gamma \kappa_s^2 < \frac{1}{4}$. Assume further that $ Tr[\cT_s^{ \nu}] <\infty$ and that 
\[   n \geq \gamma T \max\{ 1, \cN_{\cT_s}(1/\gamma T) \}  \;.   \]  
If $f_\cH \in \cD(L^s)$,  the excess risk satisfies 
\begin{align*}
 \mbe\left[ ||\cS( \bar f_T - f_\cH )||^2_{L^2} \right] &\leq c_{a,s,\kappa_s, \nu} (\gamma T)^{-1} 
  \left(\; R (\gamma T)^{-\frac{r-s}{2(a+s)}} + \frac{||L^s \tilde g_T||_\cH}{\gamma T}  \right. \\
 & +  \left.      \tilde M_T \frac{\gamma T}{n} + \tilde \sigma  \sqrt{\frac{\gamma T \cN_{\cT_s}(1/(\gamma T)) }{n}} + \tilde \tau_T \frac{\gamma T}{\sqrt n}   \;  \right) \no \\
   & \quad  +  \frac{\Sigma^2_T }{bT}  (\gamma T)^{\nu}    + \Omega^2_n(\gamma T)  \;, 
\end{align*}
where the remainder $\Omega_n(\cdot)$ is of lower order and given by 
\begin{align*}
 \Omega_n(\gamma T) 
 &= \frac{ || L^s\tilde g_T ||_\cH }{\gamma T} +   \sqrt{\frac{\gamma T}{n}}R (\gamma T)^{-\frac{r-s}{2(a+s)}} +  \gamma^2 T^5 \delta_n
\end{align*}
and where $\Sigma^2_T(R)$ is given by 
\[ \Sigma^2_T = \tilde M^2 + ||L^sf_\cH ||^2_\cH + \tilde \sigma^2  + M^2 + \sigma_T^2 \;, \]
with
\[ \sigma_T =    R(\gamma T)^{- \frac{r-s}{2(a+s)}} + \frac{||L^s f_\cH ||}{\gamma T}\;. \]
\end{theorem}

\begin{proof}[Proof of Theorem \ref{prop:newSGDPre}]
The proof follows by combining \eqref{eq:decompfinalSGD} with Proposition \ref{prop.new}, Proposition \ref{prop:preII} and Theorem \ref{theo:GDIPre} with $\alpha = 1/2$. 
\end{proof}

%




\section{Additional Results}
\label{app-additional-results}
\subsection{Supplementary Material for Section \ref{subsec:Precond}}

Here we provide additional examples of Hilbert scales, satisfying Assumption \ref{ass:link-condition2}. 

\begin{example}[Covariance Scale]
\label{lem:example-pre}
Let $\cH_K$ be an RKHS with kernel $K$ and covariance operator \eqref{eq:covariance}.  
Given $r\geq 0$ let $\cH:=Ran(\cT^r)$ and let $L=\cT^{-1}$. Then by definition \ref{eq:def-norm}  we have for any $h \in \cH$ the 
equality $||h||_\cH = ||\cT^{-r}h||_{\cH_K}$. Thus,  for any $a>0$, by \eqref{eq:iso1} 
\begin{align*}
||\cS h||_{L^2} &= || \sqrt{\cT}h||_{\cH_K} = ||\cT^{r+\frac{1}{2}} h ||_{\cH} \\
&\leq ||\cT^{r+\frac{1}{2}-a}|| \; ||\cT^a h||_{\cH} \\
&= ||\cT^{r+\frac{1}{2}-a}|| \; ||h||_{-a} \;.
\end{align*} 
Thus, \eqref{eq:link1} holds with $\bar m =  ||\cT^{r+\frac{1}{2}-a}||$, provided $0<a\leq r+\frac{1}{2}$. 

Moreover, by the same reasoning, for $\cB_s=\cS L^{-s}: \cH \to L^2(\cX, \rho_X)$ 
\begin{align*}
||\cB_s h||_{L^2} &= ||\cS L^{-s} h||_{L^2} = || \cT^{\frac{1}{2}+r+s}h ||_\cH \leq  || \cT^{\frac{1}{2}+r+s}|| \; ||h||_\cH \;.
\end{align*} 
Thus, $\cB_s$ is bounded if $s \geq -(\frac{1}{2}+r)$. This is certainly satisfied if $s\geq -a$. 
\end{example}

\[\]

\begin{example}[Example \ref{ex:pre}; Switching between different Diffusion spaces]
Let $L=e^{\frac{1}{2}\Delta}$, $t>0$ and  $\cH:=Ran(e^{-\frac{t}{2}\Delta})$ be the diffusion space, defined in Example \ref{ex:diffusion}.  
For $a>0$ we have 
\begin{align*}
||\cS h||_{L^2} &= ||e^{-\frac{t}{2}\Delta} h||_\cH \leq || e^{(-\frac{t}{2}+ \frac{a}{2})\Delta}|| \; ||h||_{-a} \;.
\end{align*} 
Hence, \eqref{eq:link1} is satisfied with $\bar m = || e^{(-\frac{t}{2}+ \frac{a}{2})\Delta}||  $, being finite if $0<a \leq t$. 
In addition, 
\begin{align*}
||\cB_s h||_{L^2} &= ||\cS L^{-s} h||_{L^2} = || e^{-\frac{1}{2}(s+t)\Delta} h||_\cH  \;.
\end{align*} 
Therefore, $\cB_s$ is finite if $s \geq -a \geq -t$. 
\end{example}

\[\]

\begin{example}[Example \ref{ex:pre2}; Switching between different orders of smoothness]
Let $s\leq 0$, $t>0$ and  $\cH_t:=Ran(e^{-\frac{t}{2}\Delta})$ be the diffusion space, defined in Example \ref{ex:diffusion}. For $u\geq 0$ define 
\[ L:=e^{\frac{1}{2}\Delta}(Id + \Delta)^{-\frac{u}{2}}\;.\] 
Then, for any $h \in \cH_t$ we have $h=e^{-\frac{t}{2}\Delta}g$ for some $g \in L^2(\cM)$ and 
\begin{align*}
 L^{-s} h &= e^{-\frac{s}{2}\Delta}(Id + \Delta)^{\frac{su}{2}}e^{-\frac{t}{2}\Delta}g \\
 &=(Id + \Delta)^{\frac{su}{2}}e^{-\frac{1}{2}(t+s)\Delta}g \;. 
\end{align*} 
Thus, $L^{-s}$ maps from the Gaussian RKHS $\cH_t$ into the larger Sobolev space $H^{-su}(\cM)$ (recall that $s \leq 0$), see Example \ref{ex:sobolev}, 
provided that $t+s \geq 0$. 

Moreover,  for $a>0$, we have 
\begin{align*}
||\cS h||_{L^2} &= ||e^{-\frac{t}{2}\Delta} h||_\cH \leq || e^{-\frac{t}{2}\Delta}L^a || \; ||h||_{-a} \;.
\end{align*} 
Since 
\[ e^{-\frac{t}{2}\Delta}L^a = (Id + \Delta)^{\frac{au}{2}}e^{-\frac{1}{2}(t-a)\Delta} \]
is bounded if $0<a\leq t$, the link condition \eqref{ass:link-condition2} is satisfied under this assumption.   
Finally, a similar calculation shows that $\cB_s = \cS L^{-s}$ is bounded if  $0\geq s \geq -a \geq -t$. Thus, Assumption \ref{ass:link-condition2} is satisfied.  
\end{example}


\subsection{Probabilistic Bounds}

In this section we collect some basic and well established probability bounds. To this end, let us introduce some notation:
\begin{equation*}
B_{n,\lam}(a,b):= \frac{1}{\sqrt \lam} \paren{ a \frac{2\kappa_s}{n\sqrt \lam } + b \sqrt{\frac{\cN_{\cT_s}(\lam)}{n \lam}} } \;, 
\end{equation*}
where $a,b>0$.

\begin{proposition}[\cite{blanchard2019lepskii}, Proposition A.1]
\label{prop:sub-proba2}
Fix $\lam >0$, $s\geq 0$, let $n \in \mbn$ and $\delta \in (0,1)$. Then with 
probability $\rho^{\otimes n}$ not less than $1-\delta$ we have 
\[  \Upsilon_s (\lam )  := || (\cT_s + \lam)^{-1/2}(\cT_s - \cT_{\bx , s})  ||_{HS} \leq 2\log(8/\delta)\sqrt{\lam } B_{n,\lam}(\kappa_s, \kappa_s) \;. \]
In particular, if  $\lam >0$ satisfies the condition  
\[ \cN_{\cT_s}(\lam)\leq n\lam \;, \]
we obtain with probability $\rho^{\otimes n}$ not less than $1-\delta$
\[   \Upsilon_s (\lam ) \leq   C_{\kappa_s}\log(8/\delta)\sqrt{\lam } \;, \]
for some $C_{\kappa_s}<\infty$. 
\end{proposition}

\[\]

\begin{proposition}[\cite{blanchard2019lepskii}, Proposition A.1]
\label{prop:sub-proba}
Let $\varphi: \mbr_+ \to \mbr_+$ be nondecreasing and sublinear. Fix $\lam >0$, let $n \in \mbn$ and $\delta \in (0,1)$. Then with 
probability $\rho^{\otimes n}$ not less than $1-\delta$ we have 
\[ \Xi ^\varphi_{\bx,s}(\lam):=||\varphi(\lam + \cT_s)\varphi(\lam + \cT_{\bx , s})^{-1}|| \leq 4\log^2(8/\delta)\paren{1+B_{n,\lam}(\kappa_s, \kappa_s)}^2 \;. \] 
If $\varphi(t)=t^r$ with $r \in [0,1]$, we have with probability $\rho^{\otimes n}$ not less than $1-\delta$ the sharper estimate 
\[ \Xi ^r_{\bx,s}(\lam):=||(\lam + \cT_s)^r (\lam + \cT_{\bx , s})^{-r}|| \leq 4^r\log^{2r}(8/\delta)\paren{1+B_{n,\lam}(\kappa_s, \kappa_s)}^{2r} \;.  \]
In particular, if  $\lam >0$ satisfies the condition  
\[ \cN_{\cT_s}(\lam)\leq n\lam \;, \]
we obtain with probability $\rho^{\otimes n}$ not less than $1-\delta$
\[ \Xi ^\varphi_{\bx,s}(\lam) \leq  C_{\kappa_s } \log^{2}(8/\delta) \;, \]
for some $ C_{\kappa_s , r} < \infty$ and 
\[ \Xi ^r_{\bx,s}(\lam) \leq  C'_{\kappa_s , r} \log^{2r}(8/\delta) \;,\]
for some $ C'_{\kappa_s , r} < \infty$. 
\end{proposition}








\checknbnotes
\checknbdrafts


\end{document}